\documentclass{article}

\usepackage{microtype}
\usepackage{graphicx}
\usepackage{subfigure}
\usepackage{booktabs} % for professional tables

\usepackage[hidelinks,colorlinks=true,linkcolor=blue,citecolor=blue]{hyperref}

\newcommand{\tb}[1]{{\textbf{#1}}}

\usepackage{amsmath}
\usepackage{amssymb}
\usepackage{physics}
\usepackage{mathtools}
\usepackage{amsthm}
\usepackage{natbib}
\usepackage{color}
\usepackage{amsfonts}       
\usepackage{url}
\usepackage{wrapfig}
\usepackage{tikz}
\usepackage{footmisc}
\usepackage{thmtools, thm-restate}
\usepackage{physics}
\mathtoolsset{showonlyrefs}
\usepackage{etoolbox}
\usepackage{makecell}
\usepackage{enumerate}
\usepackage{pifont}% http://ctan.org/pkg/pifont
\usepackage{soul}
\usepackage{ulem}
\usepackage{cancel}
\usepackage{bbm}
\theoremstyle{plain}
\newtheorem{theorem}{Theorem}%[section]

\newtheorem{lemma}{Lemma}

\theoremstyle{definition}

\newtheorem{assumption}{Assumption}[section]

\theoremstyle{plain}

\newcommand{\fS}{\mathcal{S}}
\newcommand{\fA}{\mathcal{A}}
\newcommand{\pA}{\widetilde{A}}
\newcommand{\pB}{\widetilde{b}}
\newcommand{\fY}{\mathcal{Y}}

\newcommand{\fT}{\mathcal{T}}

\newcommand{\fO}{\mathcal{O}}

\newcommand{\tw}{\widetilde{w}}
\newcommand{\R}[1][]{\mathbb{R}^{#1}}
\newcommand{\indot}[2]{{\left<#1, #2\right>}}
\newcommand{\E}{\mathbb{E}}
\newcommand{\1}{\mathbf{1}}
\newcommand{\0}{\mathbf{0}}
\newcommand{\ns}{{\abs{\fS}}}

\newcommand{\tref}[1]{\text{\ref{#1}}}

\newcommand{\col}{\operatorname{col}}

\newcommand{\avft}{\overline \fT_\lambda}
\newcommand{\inner}[2]{\langle #1, #2 \rangle}

\newtheorem{assumptionalt}{Assumption}[assumption]
\newenvironment{assumptionp}[1]{
  \renewcommand\theassumptionalt{#1}
  \assumptionalt
}{\endassumptionalt}

\allowdisplaybreaks

\usepackage[textsize=tiny]{todonotes}

\usepackage[final,nonatbib]{neurips_2025_arxiv}

\usepackage[utf8]{inputenc} % allow utf-8 input
\usepackage[T1]{fontenc}    % use 8-bit T1 fonts
\usepackage{url}            % simple URL typesetting
\usepackage{booktabs}       % professional-quality tables
\usepackage{amsfonts}       % blackboard math symbols
\usepackage{nicefrac}       % compact symbols for 1/2, etc.
\usepackage{microtype}      % microtypography
\usepackage{xcolor}         % colors

\title{Finite Sample Analysis of Linear Temporal Difference Learning with Arbitrary Features}

\author{%
  Zixuan Xie\thanks{Equal contribution} \\
  University of Virginia\\
  \texttt{xie.zixuan@email.virginia.edu} \\
  \And
  Xinyu Liu$^*$ \\
  University of Virginia\\
  \texttt{xinyuliu@virginia.edu} \\
  \And
  Rohan Chandra \\
  University of Virginia\\
  \texttt{rohanchandra@virginia.edu} \\
  \And
  Shangtong Zhang \\
  University of Virginia\\
  \texttt{shangtong@virginia.edu} \\
}

\begin{document}

\maketitle

\begin{abstract}
Linear TD($\lambda$) is one of the most fundamental reinforcement learning algorithms for policy evaluation. Previously, convergence rates are typically established under the assumption of linearly independent features, which does not hold in many practical scenarios. This paper instead establishes the first $L^2$ convergence rates for linear TD($\lambda$) operating under arbitrary features, without making any algorithmic modification or additional assumptions. Our results apply to both the discounted and average-reward settings. To address the potential non-uniqueness of solutions resulting from arbitrary features, we develop a novel stochastic approximation result featuring convergence rates to the solution set instead of a single point.
\end{abstract}

\section{Introduction}
% Temporal Difference (TD) learning is a cornerstone algorithm in reinforcement learning, particularly for policy evaluation. While TD learning has shown remarkable empirical success, its theoretical analysis presents significant challenges, especially when considering finite-time guarantees rather than asymptotic results.

% Previous work has established convergence properties of TD learning under the assumption that the feature matrix in value function approximation is full ranked\cite{tsitsiklis1997analysis}\cite{mitra2024simplefinitetimeanalysistd}. However, this condition is often violated in practical applications, where feature matrices may be rank-deficient due to correlations or redundancies in the state representation.

% Our work addresses this limitation by extending the finite-time analysis of TD learning to handle arbitrary feature matrices. We employ novel analytical techniques based on careful Lyapunov function construction and decomposition of the learning dynamics. The analysis covers both the basic TD(0) algorithm and its generalization TD($\lambda$), providing explicit convergence rates under both i.i.d. and Markovian sampling conditions.

% The significance of this work lies in bridging the gap between theoretical guarantees and practical applications, where rank conditions on feature matrices cannot always be guaranteed. Our results provide practitioners with more realistic performance guarantees while maintaining the mathematical rigor of previous analyses.
Temporal difference learning (TD, \citet{sutton1988learning}) is a fundamental algorithm in reinforcement learning (RL, \citet{sutton2018reinforcement}), enabling efficient policy evaluation by combining dynamic programming \citep{bellman1966dynamic} with stochastic approximation (SA, \citet{benveniste1990MP,kushner2003stochastic,borkar2009stochastic}). 
Its linear variant, linear TD($\lambda$) \citep{sutton1988learning}, emerges as a practical extension, employing linear function approximation to tackle large or continuous state spaces where tabular representations become impractical. Linear TD($\lambda$) takes the dot product between features and weights to compute the approximated value. 
Establishing theoretical guarantees for linear TD($\lambda$), particularly convergence rates, has been a major focus of research. 
Most existing works (Table~\ref{tab:linear_td_comparison}),
however,
require the features used in linear TD to be linearly independent.
As argued in \citet{wang2024sureconvergencelineartemporal},
this assumption is impractical in many scenarios.
For example, 
in continual learning with sequentially arriving data \citep{ring1994continual, khetarpal2022towards,abel2023definition},
there is no way to rigorously verify whether the features are independent or not.
See \citet{wang2024sureconvergencelineartemporal} for more discussion on the restrictions of the feature independence assumptions. 
Furthermore, \citet{dayan1992convergence,tsitsiklis1997analysis,tsitsiklis1999average} also outline the elimination of the linear independence assumption as a future research direction.

While efforts have been made to eliminate the linear independence assumption \citep{wang2024sureconvergencelineartemporal},
they only provide asymptotic (almost sure) convergence guarantees in the discounted setting.
By contrast,
this paper establishes \tb{the first $L^2$ convergence rates for linear TD($\lambda$) with arbitrary features in both discounted and average-reward settings}.
This success is enabled by a novel stochastic approximation result (Theorem~\tref{thm:sa}) concerning the convergence rates to a solution set instead of a single point,
driven by a novel Lyapunov function.
%  specifically developed to derive $L^2$ convergence rates for Linear TD($\lambda$) under arbitrary features. This framework is designed to handle the core challenge posed by feature dependence, which is analyzing convergence to the solution set, rather than a single point. Instead of targeting a specific solution, we measure the $L^2$  distance between the algorithm’s iterates and this solution set, providing a more comprehensive assessment of the algorithm's convergence behavior. 
This new result provides a unified approach applicable to both discounted (Theorem~\tref{thm:td markov}) and average-reward (Theorem~\tref{thm:ar td markov}) settings. 
Notably, we do not make any algorithmic modification and do not introduce any additional assumptions.
% Notably, when addressing the non-uniqueness of the solution set in the average-reward setting, characterized by an entire affine subspace of valid solutions that differ only by constant shifts in a particular direction, we adapt and extend the subspace definition and projection methods from \citet{zhang2021finite} to make these tools compatible with our arbitrary feature analysis.
Table~\tref{tab:linear_td_comparison} provides a detailed comparison of existing theoretical analyses for linear TD($\lambda$), contextualizing our contributions within the landscape of prior work.

% Foundational analyses provide asymptotic convergence properties for both discounted and average-reward settings \citep{tsitsiklis1997analysis,tsitsiklis1999average}, often requiring the assumption of linearly independent features. 
% Key examples include seminal work establishing convergence for the discounted setting \citep{tsitsiklis1997analysis} and for the average-reward setting \citep{tsitsiklis1999average}. 
% Subsequent finite-sample studies focus on deriving convergence rates, such as \citet{bhandari2018finite} for the discounted case and \citet{zhang2021finite} for the average-reward case, but generally maintain this feature independence requirement, which greatly simplifies the theoretical analysis.
% However, this assumption is often unrealistic, as verifying feature independence is frequently infeasible in practice, for instance, in continual learning scenarios with sequentially arriving data \citep{ring1994continual, khetarpal2022towards,abel2023definition}.  
% Indeed, the need to analyze convergence beyond the linear independence assumption is recognized early. \citet{dayan1992convergence} provides an initial asymptotic convergence analysis for TD($\lambda$) without rates in this setting. While recent work by \citet{wang2024sureconvergencelineartemporal} makes significant progress by establishing almost sure convergence for discounted Linear TD($0$) under arbitrary features, finite-sample convergence rates are still not provided.
% This strongly motivates the need for analysis beyond this restrictive assumption to quantify the convergence speed.
\begin{table}[h!]
    \centering
    \begin{tabular}{c|c|c|c|c}
    \toprule
     & Setting & Features & \makecell{Noise \\ Type}  & Rate \\
    \midrule
    %\citet{dayan1992convergence} & $\gamma < 1$ & Arbitrary & Markovian & \\
    \citet{tsitsiklis1997analysis} & $\gamma < 1$ & Independent & Markovian &  \\
    \citet{bhandari2018finite} & $\gamma < 1$& Independent & Markovian & $\checkmark$ \\
    \citet{lakshminarayanan2018linear} & $\gamma < 1$ & Independent & i.i.d. & $\checkmark$ \\
    \citet{srikant2019finite} & $\gamma < 1$& Independent & Markovian & $\checkmark$ \\
    \citet{wang2024sureconvergencelineartemporal} & $\gamma < 1$ & Arbitrary & Markovian & \\
    \citet{chen2025concentration} & $\gamma < 1$  & Independent & i.i.d. & $\checkmark$ \\ 
    \citet{mitra2024simple}& $\gamma < 1$& Independent & Markovian & $\checkmark$ \\
    % \citet{wu2025unifying} & $\gamma < 1$& Arbitrary & Markovian & \\
    Theorem~\tref{thm:td markov} & $\gamma < 1$  & Arbitrary & Markovian & $\checkmark$ \\
    \midrule
    \citet{tsitsiklis1999average} & $\gamma = 1$ & Independent & Markovian & \\
    \citet{zhang2021finite} & $\gamma = 1$ & Independent & Markovian & $\checkmark$ \\
    % \citet{blaser2025asymptotic} & $\gamma = 1$ & Independent & Markovian &\\
    \citet{chen2025nonasymptotic}& $\gamma = 1$& Independent & Markovian & $\checkmark$ \\
    Theorem~\tref{thm:ar td markov} & $\gamma = 1$ & Arbitrary & Markovian & $\checkmark$ \\
    \bottomrule
    \end{tabular}
    \caption{Comparison of finite-sample analyses for linear TD($\lambda$).
        ``Setting'' indicates the problem setting: $\gamma < 1$ stands for the discounted setting and $\gamma = 1$ stands for the average reward setting.
      ``Features'' describes assumptions on the features. ``Independent'' indicates linear independence is assumed. ``Arbitrary'' indicates no assumption is made on features.
      ``Noise Type'' indicates the data generation process: Markovian samples or independent and identically distributed (i.i.d.) samples.
      ``Rate'' is checked if a convergence rate is provided.
    }
    \label{tab:linear_td_comparison}
\end{table}

\section{Background}
\label{sec:background}
\tb{Notations.}
We use $\langle x,y \rangle \doteq x^\top y$ to denote the standard inner product in Euclidean spaces
and $\norm{\cdot}$ to denote the $\ell_2$ norm for vectors and the associated induced operator norm (i.e., the spectral norm) for matrices, unless stated otherwise. 
A function $f$ is said to be $L$-smooth (w.r.t. $\norm{\cdot}$) if $\forall w, w'$, $f(w') \leq f(w) + \langle\nabla f(w), w' - w\rangle + \frac{L}{2} \norm{w' - w}^2$. For a matrix $A$, $\col(A)$ denotes its column space, $\ker(A)$ denotes its kernel, and $A^\dagger$ denotes its Moore-Penrose inverse. When $x$ is a point and $U$ is a set, we denote $d(x,U) \doteq \inf_{y\in U} \norm{x - y}$ as the Euclidean distance from $x$ to $U$.
For sets $U, V$, their Minkowski sum is $U+V \doteq \{u+v \mid u \in U, v \in V\}$; and $U^\perp$ denotes the orthogonal complement of $U$.
We use $\0$ and $\1$ to denote the zero vector and the all-ones vector respectively, where the dimension is clear from context.
For any square matrix $A \in \R[d \times d]$ (not necessarily symmetric),
we say $A$ is negative definite (n.d.) if there exists a $\xi > 0$ such that $x^\top A x \leq -\xi \norm{x}^2 \ \forall x \in \R[d]$.
For any set $E \subseteq \R[d]$,
we say $A$ is \underline{n.d. on $E$} if there exists a $\xi > 0$ such that $x^\top A x \leq -\xi \norm{x}^2 \, \forall x \in E$.
A is negative semidefinite (n.s.d.) if $\xi = 0$ in the above definition.
% For a diagonal matrix $D$, we define $\norm{x}_D \doteq \sqrt{x^\top D x}$ as the induced norm. 
 
\tb{Markov Decision Processes.} We consider an infinite horizon Markov Decision Process (MDP, \citet{bellman1957markovian})
defined by a tuple $(\fS, \fA, p, r, p_0)$, where $\fS$ is a finite set of states, $\fA$ is a finite set of actions, $p: \fS \times \fS \times \fA \to \qty[0,1]$ is the transition probability function, $r: \fS \times \fA \to \R$ is the reward function, 
and $p_0: \fS \to [0, 1]$ denotes the initial distribution.
In this paper, we focus on the policy evaluation problem, where the goal is to estimate the value function of an arbitrary policy $\pi: \fA \times \fS \to [0, 1]$.
At the time step $0$,
an initial state $S_0$ is sampled from $p_0$.
At each subsequent time step $t$, the agent observes state $S_t \in \fS$,
executes an action $A_t \sim \pi(\cdot | S_t)$, receives reward $R_{t+1} \doteq r(S_t, A_t)$, and transitions to the next state $S_{t+1} \sim p(\cdot | S_t, A_t)$. 
We use $P_\pi$ to denote the state transition matrix induced by the policy $\pi$, i.e., $P_\pi\qty[s,s']=\sum_{a\in\fA}\pi(a|s)p(s'|s,a)$.
% $P_\pi\qty[(s,a),(s',a')]=p(s'|s,a)\pi(a'|s')$. 
Let $d_\pi \in \mathbb{R}^{|\fS|}$ be the stationary distribution of the Markov chain induced by the policy $\pi$.
%  where $d_\pi(S)$ is the stationary probability of being in state $S \in \fS$
 We use $D_\pi$ to denote the diagonal matrix whose diagonal is $d_\pi$.
% \begin{align}
    % \label{eq:P pi}
    % \textstyle P_\pi\qty[(s,a),(s',a')]=p(s'|s,a)\pi(a'|s').
% \end{align}

\tb{Linear Function Approximation.}
In this paper,
we use linear function approximation to approximate value functions $v_\pi: \fS \to \R$ (to be defined shortly). 
% In general, we aim to approximate a state-value function $v_\pi$, which we will define separately for different cases afterward, using a linear combination of features.
We consider a feature mapping $x: \fS \to \R[d]$ and a weight vector $w \in \R[d]$.
% The function $x$ maps each state $s \in \fS$ to a $d$-dimensional real vector $x(s)$.
We then approximate $v_\pi(s)$ with $x(s)^\top w$.
% We can represent the feature function $x$ more compactly as a matrix $X \in \R[\ns \times d]$, where the $s$-th row of $X$ is $x(s)^\top$, and the approximated state-value function as a vector $Xw$.
We use $X \in \R[\ns \times d]$ to denote the feature matrix, 
where the $s$-th row of $X$ is $x(s)^\top$. The approximated state-value function across all states can then be represented as the vector $Xw \in \R[\ns]$.
The goal is thus to find a $w$ such that $Xw$ closely approximates $v_\pi$.

% In the following, we denote $x(S_{t})$ by $x_t$ for simplicity.

\tb{Discounted Setting.}
In the discounted setting, we introduce a discount factor $\gamma \in [0, 1)$.
%  that gives higher weight to immediate rewards compared to rewards received in the future.
The (discounted) value function $v_\pi : \fS \to \R$ for policy $\pi$ is defined as
% \begin{align}
  $\textstyle v_\pi(s) \doteq \E
  \qty[\sum_{i=0}^\infty \gamma^i R_{t+i+1} \middle| S_t = s]$.
% \end{align}
We define the Bellman operator $\fT: \R[\ns] \to \R[\ns]$ as $\fT v \doteq r_\pi + \gamma P_\pi v$, 
where $r_\pi \in \mathbb{R}^{\ns}$ is the vector of expected immediate rewards under $\pi$, with components $r_\pi(s) = \sum_a \pi(a|s)r(s, a)$.
With a $\lambda \in [0, 1]$,
the $\lambda$-weighted Bellman operator $\fT_\lambda$ is defined as
$\textstyle \fT_\lambda v \doteq (1 - \lambda) \sum_{m=0}^{\infty} \lambda^m \fT^{m+1}v = r_{\lambda} + \gamma P_\lambda v$,
where
\begin{align}
  \textstyle r_\lambda =& \textstyle \sum_{k=0}^{\infty} (\lambda \gamma)^k P_\pi^k r_\pi = (1 - \gamma \lambda P_\pi)^{-1} r_\pi, \\  
  \textstyle P_\lambda =& \textstyle (1 - \lambda) \sum_{m=0}^{\infty} (\lambda \gamma)^m P_\pi^{m+1} = (1 - \lambda)(1 - \gamma \lambda P_\pi)^{-1} P_\pi.
\end{align}
This represents a weighted average of multi-step applications of $\fT$.
It is well-known that $v_\pi$ is the unique fixed point of $\fT_\lambda$ \citep{bertsekas1996neuro}.
Linear TD($\lambda$) is a family of TD learning algorithms that use eligibility traces to estimate $v_\pi(s)$ of the fixed policy $\pi$ with linear function approximation. 
% It bridges between one-step TD methods (like TD(0)) and Monte Carlo methods. 
The algorithm maintains a weight vector $w_t \in \R[d]$ and an eligibility trace vector $e_t \in \R[d]$, with the following update rules:
\begin{align}
    &w_{t+1} = w_t + \alpha_t (R_{t+1} + \gamma x(S_{t+1})^\top w_t - x(S_t)^\top w_t) e_t, \\
\label{eq:td lambda}
    &e_t = \gamma \lambda e_{t-1} + x(S_t), \, e_{-1} = \0.
    \tag{Discounted TD}
\end{align}
%  where $\qty{\alpha_t}$ is the learning rate, $\lambda \in [0, 1)$ is the trace decay parameter. \sz{Do we need $\lambda < 1$?} 
Here, $\qty{\alpha_t}$ is the learning rate. The eligibility trace $e_t$ tracks recently visited states, assigning credit for the prediction error to multiple preceding states.
Let
\begin{align}
    A \doteq X^\top D_\pi(\gamma P_\lambda - I)X,\, b \doteq X^\top D_\pi r_\lambda, \, W_* \doteq \qty{w | Aw + b = \0}.
\end{align}
If $X$ has a full column rank, \citet{tsitsiklis1997analysis} proves that $W_*$ is a singleton and $\qty{w_t}$ converge to $-A^{-1}b$ almost surely.
A key result used by \citet{tsitsiklis1997analysis} is that the matrix $D_\pi(\gamma P_\lambda - I)$ is n.d. \citep{sutton1988learning}.
As a result, the $A$ matrix is also n.d. when $X$ has a full column rank.
\citet{wang2024sureconvergencelineartemporal} prove, without making any assumption on $X$, that
$W_*$ is always nonempty and the $\qty{w_t}$ converges to $W_*$ almost surely.
A key challenge there is that without making assumptions on $X$, 
$A$ is only n.s.d.

% TD($\lambda$) aims to find a weight vector $w$ such that $Xw$ is a fixed point of $\fT^{(\lambda)}$, projected onto the representable subspace. This yields the projected Bellman equation
% We expect~\eqref{eq:td lambda} to convergence to a solution of the following projected Bellman equation
% \begin{align}
    % \textstyle Xw = \Pi_{D,X} \fT^{(\lambda)}(Xw), \label{eq:proj_bellman_discounted}
% \end{align}
% where $\Pi_{D,X} v \doteq \arg \min_{v_0 \in \col(X)} \norm{v_0 - v}^2_D$ and $\norm{x}_D \doteq \sqrt{x^\top Dx}$. 
% According to Lemma 1 of \citet{wang2024sureconvergencelineartemporal}, we have $\Pi_{D,X} = X(D^{1/2}X)^{\dagger}D^{1/2}$.
% Notably,
% here the Moore-Penrose inverse is used because no assumption is made on $X$.
% If linearly independent features are assumed (i.e., $X$ has a full column rank),
% $\Pi_{D, X}$ can indeed be defined with the standard matrix inverse.

\tb{Average-Reward Setting.}
% While the discounted setting is appropriate for finite-horizon tasks, the average-reward setting provides an alternative framework better suited for continuing tasks without natural termination. 
% While the discounted setting suits infinite-horizon tasks, the average-reward setting is better for continuing tasks without termination, 
In the average-reward setting,
the overall performance of a policy $\pi$ is measured by the average reward
% aiming to maximize the average reward per time step, denoted by $J_\pi$ under policy $\pi$:
$J_\pi \doteq \lim_{T \to \infty} \frac{1}{T} \E\left[\sum_{t=0}^{T-1} R_{t}\right]$. 
% \begin{align}
% \textstyle J_\pi \doteq \lim_{T \to \infty} \frac{1}{T} \E\left[\sum_{t=0}^{T-1} R_{t}\right].
% \end{align}
% We assume $J_\pi$ exists and is independent of the initial state. The differential value function $v_\pi : \fS \to \R$ represents the difference in long-term reward starting from state $s$, compared to the average reward $J_\pi$. It satisfies the Bellman equation
The corresponding (differential) value function is defined as
% \begin{align}
% \label{eq:diff_value_bellman}
$\textstyle \overline v_\pi(s) = \lim_{T\to\infty} \frac{1}{T}\sum_{i=0}^{T-1} \E \qty[(r(S_{t+i}, A_{t+i}) - J_\pi) \middle | S_t = s]$.
%   S_t = s; 
%   A_t \sim \pi(\cdot | S_t); 
%   S_{t+1} \sim p(\cdot | S_t, A_t)].
% \end{align}
We define the Bellman operator $\overline \fT: \R[\ns] \to \R[\ns]$ as $\overline \fT v \doteq r_\pi - J_\pi\1 + P_\pi v$.
Similarly,
the $\lambda$-weighted counterpart $\overline \fT_\lambda$ is defined as
$\avft v \doteq r_\lambda - \frac{J_\pi}{1 - \lambda} \1 + P_\lambda v$.
% Here we have abused $v_\pi$ to denote both the discounted and differential value function,
% though we expect no confusion and the meaning of $v_\pi$ should be clear from the context.
Although $\overline v_\pi$ is a fixed point of $\overline \fT_\lambda$,
it is not the unique fixed point.
In fact, 
\begin{align}
    \label{eq ar fixed point}
  \textstyle \qty{\overline v_\pi + c \1 \mid c \in \R}  
\end{align}
are all the fixed points of $\overline \fT_\lambda$ \citep{puterman2014markov}.
% Notably, $\overline v_\pi$ is only unique up to an additive constant. If $v_\pi$ satisfies the Bellman equation, then $v_\pi + c\1$ is also a valid solution for any constant $c$.
Linear average-reward TD($\lambda$) is an algorithm for estimating both $J_\pi$ and $\overline v_\pi$ using linear function approximation and eligibility traces. The update rules are
\begin{align}
    e_t &= \lambda e_{t-1} + x(S_t), \, e_{-1} = \0,\\
    w_{t+1} &= w_t + \alpha_t (R_{t+1} - \hat{J}_t + x(S_{t+1})^\top w_t - x(S_t)^\top w_t) e_t, \notag\\
    \hat{J}_{t+1} &= \hat{J}_t + \beta_t (R_{t+1} - \hat{J}_t),
    \label{eq:artd}
    \tag{Average Reward TD}
\end{align}
where $\qty{\alpha_t}$ and $\qty{\beta_t}$ are learning rates. 
Let
\begin{align}
\label{eq:bar W*}
    \textstyle \overline A \doteq X^\top D_\pi (P_\lambda - I)X, \, \overline b \doteq X^\top D_\pi (r_\lambda - \frac{J_\pi}{1 - \lambda} \1), \, \overline W_* \doteq \qty{w | \overline A w + \overline b = \0}.
\end{align}
If $X$ has a full column rank and $\1 \notin \col(X)$,
\citet{tsitsiklis1999average} proves that $\overline W_*$ is a singleton and $\qty{w_t}$ converge to $-\overline A^{-1} \overline b$ almost surely.
This is made possible by an important fact from the Perron-Frobenius theorem (see, e.g., \citet{seneta2006non}) that 
\begin{align}
    \label{eq nsd solution}
  \qty{w | w^\top D_\pi (P_\lambda - I) w = 0} = \qty{c\1 | c\in \R}.
\end{align}
\citet{zhang2021finite} further provides a convergence rate, still assuming $X$ has a full column rank but without assuming $\1 \notin \col(X)$.
When $X$ does not have a full column rank, 
to our knowledge,
it is even not clear whether $\overline W_*$ is always nonempty or not,
much less the behavior of $\qty{w_t}$.
% Similarly,
% We expect~\eqref{eq:artd} to converge to a solution of the following projected Bellman equation
% To analyze this algorithm, 
% The algorithm seeks a weight vector $w$ such that $Xw$ is a fixed point of the projected $\lambda$-weighted Bellman operator:
% \begin{align}
    % Xw = \Pi_{D,X} \overline \fT^{(\lambda)}(Xw). \label{eq:proj_bellman_average}
% \end{align}
% In the rest of the paper,
% we use linear TD($\lambda$) to denote both~\eqref{eq:td lambda} and~\eqref{eq:artd} when it does not confuse.
% However, due to the non-uniqueness of the differential value function, the solution to this equation may not be unique, which introduces additional challenges for convergence analysis.

\section{Main Results}
\label{thms}
We start with our assumptions. As promised, we do not make any assumption on $X$.
\begin{assumption}
\label{assum:markov}
    The Markov chain associated with $P_\pi$ is irreducible and aperiodic. 
\end{assumption}
% This is a common assumption to analyze time-inhomogeneous Markovian noise \citep{zhang2022globaloptimalityfinitesample}. 
\begin{assumptionp}{LR}
  \label{assu lr}
  The learning rates are $\alpha_t = \frac{\alpha}{(t + t_0)^{\xi}}$ and $\beta_t =c_\beta\alpha_t$,  
  where $\xi \in (0.5,1]$, $\alpha>0$, $t_0>0$, and $c_\beta>0$ are constants.
\end{assumptionp}
\paragraph{Discounted Setting.}
\citet{wang2024sureconvergencelineartemporal} proves the almost sure convergence of \eqref{eq:td lambda} with arbitrary features by using $\norm{w - w_*}^2$ with an arbitrary and fixed $w_* \in W_*$ as a Lyapunov function and analyzing the property of the ODE $\dv{w(t)}{t} = Aw(t)$.
Since $A$ is only n.s.d.,
\citet{wang2024sureconvergencelineartemporal} conducts their analysis in the complex number field.
In this work,
instead of following the ODE-based analysis originating from \citet{tsitsiklis1997analysis,borkar2000ode},
we extend \citet{srikant2019finite} to obtain convergence rates
by using $d(w, W_*)^2$ as the Lyapunov function.
To our knowledge,
this is the first time that such distance function to a set is used as the Lyapunov function to analyze RL algorithms,
which is our key technical contribution from the methodology aspect.
According to Theorem 1 of \citet{wang2024sureconvergencelineartemporal}, $W_*$ is nonempty, and apparently convex and closed.\footnote{This theorem only discusses the case of $\lambda = 0$. The proof for a general $\lambda \in [0, 1]$ is exactly the same up to change of notations.}
Let $\Gamma(w) \doteq \arg\min_{w_*\in W_*} \norm{w - w_*}$ be the orthogonal projection to $W_*$.
We then define
    $\textstyle L(w) \doteq \frac{1}{2} d(w, W_*)^2 = \frac{1}{2}\norm{w - \Gamma(w)}^2$.
Two important and highly non-trivial observations are 
\begin{itemize}
    \item[(i)] $\nabla L(w) = w - \Gamma(w)$ (Example 3.31 of \citet{beck2017first}),
    \item[(ii)] $L(w)$ is 1-smooth w.r.t. $\norm{\cdot}$ (Example 5.5 of \citet{beck2017first}).
\end{itemize}
Both (i) and (ii) result from the fact that $W_*$ is nonempty, closed, and convex.
Using $L(w)$ as the Lyapunov function together with more characterization of $\nabla L(w)$ (Section~\ref{sec TD rate}),
we obtain
% \label{subsec td lambda}
% Without making any assumption on $X$, 
% \eqref{eq:proj_bellman_discounted} may have multiple solutions.
% Linear TD($\lambda$) with linearly dependent features faces the issue that \eqref{eq:proj_bellman_discounted} may have multiple solutions. 
% We define $W^* \doteq \qty{w \in \R[d] : Xw = \Pi_{D,X} \fT^{(\lambda)}(Xw) }\subseteq \R[d]$ as the set of all weight vectors that satisfy \eqref{eq:proj_bellman_discounted}. 
%  
% Then, by using $d(w, W^*)^2$ as a Lyapunov function, we obtain
\begin{theorem}
    % [$L^2$ Convergence Rate of~\eqref{eq:td lambda}]
\label{thm:td markov}
Let Assumptions~\tref{assum:markov} and \tref{assu lr} hold and $\lambda \in [0, 1]$.
Then for sufficiently large $t_0$ and $\alpha$, there exist some constants $C_\text{Thm\ref{thm:td markov}}$ and $\kappa_\tref{thm:td markov}\doteq \alpha C_\tref{prop:a}>1$ such that the iterates $\qty{w_t}$ generated by~\eqref{eq:td lambda} satisfy for all $t$
\begin{align*}
    \textstyle \E\qty[d(w_t, W_*)^2] \leq C_\text{Thm\ref{thm:td markov}}\qty(\qty(\frac{t_0}{t})^{\lfloor \kappa_\tref{thm:td markov} \rfloor}d(w_0, w_*)^2 + \qty(\frac{\ln (t+t_0)}{(t+t_0)^{\min(2\xi-1, \lfloor \kappa_\tref{thm:td markov} \rfloor-1)}})).
\end{align*}
\end{theorem}
The proof is in Section~\ref{sec TD rate}. Notably, Lemma~3 of \citet{wang2024sureconvergencelineartemporal} states that
for any $w_*, w_{**} \in W_*$,
it holds that $Xw_* = Xw_{**}$.
We then define 
\begin{align}
    \label{eq limiting value}
  \hat v_\pi \doteq Xw_*  
\end{align}
for any $w_* \in W_*$.
Theorem~\ref{thm:td markov} then also gives the $L^2$ convergence rate of the value estimate, i.e., the rate at which $Xw_t$ converges to $\hat v_\pi$.
The value estimate $\hat v_\pi$ is the unique fixed point of a projected Bellman equation. 
See \citet{wang2024sureconvergencelineartemporal} for more discussion on the property of $\hat v_\pi$.
Additionally, 
by choosing a sufficiently large $\alpha$, we can ensure $\lfloor\kappa_\tref{thm:td markov}\rfloor - 1 \ge 2\xi - 1$, so the rate is determined by the exponent $2\xi - 1$. For the standard choice $\xi=1$, the resulting rate becomes $\mathcal{O}(\ln t/t)$, which matches existing analyses that assume 
linearly independent features \citep{bhandari2018finite,srikant2019finite}.
An analogous observation holds for Theorems~\ref{thm:ar td markov} and 
\ref{thm:sa} as well, since their corresponding $\kappa$ 
is also proportional to $\alpha$.

\paragraph{Average Reward Setting.}
Characterizing $\overline W_*$ is much more challenging.
We first present a novel decomposition of the feature matrix $X$.
To this end,
define $m \doteq \rank(X)\leq \min\qty{\ns, d}$.
If $m = 0$, all the results in this work are trivial and we thus discuss only the case $m \geq 1$.
\begin{lemma}
\label{lem:z1z2}
There exist matrices \( X_1, X_2 \) such that $X = X_1 + X_2$ with the following properties (1) $\rank(X_1) = m - \mathbb{I}_{\1 \in \col(X)}$ and $\1 \notin X_1$
    % Assume $\text{rank}(X) = m \leq \min(\ns, d)$, then . Here, (1) $\col(X_1) = \operatorname{span}\{ x_1, \ldots, x_{m-1} \}$, where \( \{ x_1, \ldots, x_{m-1} \} \) is a linearly independent subset of \( \col(X) \), and \( \1 \notin \col(X_1) \); 
    % (2) \( X_2 = [\xi_1 \1, \ldots, \xi_d \1] \) with \( \xi_i \in \mathbb{R} \).
    (2) $X_2 = \1 \theta^\top$ with $\theta \in \R[d]$.
    % , and \( \theta_1 = \cdots = \theta_{m-1} = 0 \).
\end{lemma}
The proof is in Section~\ref{proof:z1z2} with $\mathbb{I}$ being the indicator function.
Essentially, $X_2$ is a rank one matrix with identical rows $\theta$ (i.e., the $i$-th column of $X_2$ is $\theta_i \1$). 
To our knowledge,
this is the first time that such decomposition is used to analyze average-reward RL algorithms,
which is our second technical contribution from the methodology aspect.
This decomposition is useful in three aspects.
First, we have $\overline A = X_1^\top D_\pi (P_\lambda - I)X_1$ (Lemma~\tref{lem:A bar}).
Second, this decomposition is the key to prove that $\overline W_*$ is nonempty (Lemma~\tref{lem:w*}).
Third, this decomposition is the key to characterize $\overline W_*$ in that $\overline W_* = \qty{\overline w_*} + \ker(X_1)$ with $\overline w_*$ being any vector in $\overline W_*$ (Lemma~\tref{lem:fix_points}).
To better understand this characterization,
we note that $\ker(X_1) = \qty{w | Xw = c\1, c \in \R}$ (Lemma~\tref{lem:fix_points}).
As a result, 
adding any $w_\0 \in \ker(X_1)$ to a weight vector $w$ changes the resulting value function $Xw$ only by $c\mathbf{1}$.
Two values $v_1$ and $v_2$ can be considered ``duplication'' if $v_1 - v_2 = c\1$ (cf.~\eqref{eq ar fixed point}).
So intuitively, $\ker(X_1)$ is the source of the ``duplication''. 
With the help of this novel decomposition, we obtain
\begin{theorem}
    % [$L^2$ Convergence Rate of~\eqref{eq:artd}]
\label{thm:ar td markov}
Let Assumptions~\tref{assum:markov} and \tref{assu lr} hold and $\lambda\in[0,1)$.
Then for sufficiently large $\alpha$, $t_0$ and $c_\beta$, there exist some constants $C_\text{Thm\ref{thm:ar td markov}}$ and $\kappa_\tref{thm:ar td markov}\doteq \alpha C_\tref{lem negdef pA}>1$ such that the iterates $\qty{w_t}$ generated by~\eqref{eq:artd} satisfy for all $t$
\begin{align*}
    \textstyle \E\qty[(\hat{J}_t - J_\pi)^2 + d(w_t, \overline W_*)^2] \leq&\textstyle  C_\text{Thm\ref{thm:ar td markov}}\qty(\frac{t_0}{t})^{\lfloor \kappa_\tref{thm:ar td markov} \rfloor}\qty[(\hat{J}_0 - J_\pi)^2 + d(w_0,\overline W_*)^2] \\
    &\textstyle + C_\text{Thm\ref{thm:ar td markov}}\qty(\frac{\ln (t+t_0)}{(t+t_0)^{\min(2\xi-1, \lfloor \kappa_\tref{thm:ar td markov} \rfloor-1)}}).
\end{align*}
\end{theorem}
The proof is in Section~\ref{sec:ar}.

\paragraph{Stochastic Approximation.}
We now present a general stochastic approximation result to prove Theorems~\tref{thm:td markov} and~\tref{thm:ar td markov}.
The notations in this part are independent of the rest of the paper.
We consider a general iterative update rule for a weight vector $w \in \R[d]$, driven by a time-homogeneous Markov chain $\qty{Y_t}$ evolving in a possibly infinite space $\fY$:
\begin{equation}
    \tag{SA}
w_{t+1} = w_t + \alpha_t H(w_t, Y_{t+1}),
\label{eq:Q_update_H_background}
\end{equation}
where $H: \R[d] \times \fY \to \R[d]$ defines the incremental update.
% We first assume Lipschitz continuity.
\begin{assumptionp}{A1}
    \label{asp:lipH}
        There exists a constant $C_\tref{asp:lipH}$ such that $\sup_{y \in \fY} \norm{H(0, y)} < \infty$,
       \begin{align*}
           \norm{H(w_1,y)-H(w_2,y)} \leq& C_\tref{asp:lipH} \norm{w_1-w_2} \qq{$\forall w_1, w_2, y$.}
       \end{align*}
\end{assumptionp}
% We then assume the geometric mixing of functions on $\qty{Y_t}$. 
\begin{assumptionp}{A2}
    \label{asp:mixing}
    $\qty{Y_t}$ has a unique stationary distribution $d_\fY$.
    % $\qty{Y_t}$ forms a time-homogeneous Markov chain on $\mathcal{Y}$, with transition kernel $P(y,A)$, where $y \in \mathcal{Y}$, $A\subseteq \mathcal{Y}$ is a measurable set, and $P(y,A) = \Pr(Y_{t+1}\in A|Y_t=y)$. 
    % We assume $\qty{Y_t}$ has a unique stationary distribution $d_\mathcal{Y}$.
    % Let $$h(w) \doteq \E_{y\sim d_\mathcal{Y}}\qty[H(w, y)].$$
        % where $P^n$ is the $n$-step transition kernel.
\end{assumptionp}
% Notably, $\fY$ does not need to be finite.
% Assumption~\ref{asp:mixing} allows us to define the expected update 
Let $h(w) \doteq \E_{y \sim d_\mathcal{Y}}\qty[H(w, y)]$.
% Here $h(w)$ defines the direction of the expected updates 
Assumption~\ref{asp:lipH} then immediately implies that
\begin{align}
    \norm{h(w_1)-h(w_2)} \leq& C_\tref{asp:lipH} \norm{w_1-w_2} \qq{$\forall w_1, w_2$.}
\end{align}
In many existing works about stochastic approximation \citep{borkar2000ode,chen2023lyapunov,qian2024almost,borkar2025ode}, it is assumed that $h(w) = 0$ adopts a unique solution.
To work with the challenges of linear TD with arbitrary features,
we relax this assumption and consider a set $W_*$.
Importantly, $W_*$ does not need to contain all solutions to $h(w) = 0$.
Instead, we make the following assumptions on $W_*$.
\begin{assumptionp}{A3}
\label{asp:W*} 
$W_*$ is nonempty, closed, and convex. 
\end{assumptionp}
Notably, $W_*$ does not need to be bounded. 
Assumption~\tref{asp:W*} ensures that the orthogonal projection to $W_*$ is well defined,
allowing us to define
$\Gamma(w) \doteq \arg\min_{w_*\in W_*} \norm{w - w_*}, L(w) \doteq \frac{1}{2}\norm{w - \Gamma(w)}^2$.
As discussed before, Assumption~\tref{asp:W*} ensures that $\nabla L(w) = w - \Gamma(w)$ and $L$ is 1-smooth w.r.t. $\norm{\cdot}$ \citep{beck2017first}.
We further assume that the expected update $h(w_t)$ decreases $L(w_t)$ in the following sense, making $L(w)$ a candidate Lyapunov function.
\begin{assumptionp}{A4}
\label{asp:negdrift}
There exists a constant $C_\tref{asp:negdrift} > 0$ such that almost surely,
\begin{align*}
    \langle \nabla L(w_t), h(w_t) \rangle \leq -C_\tref{asp:negdrift} L(w_t).
\end{align*}
\end{assumptionp}
Lastly, we make the most ``unnatural'' assumption of $W_*$.
\begin{assumptionp}{A5}
    \label{asp:X}
There exists a matrix $X$ and constants $C_\tref{asp:X}$ and $\tau \in [0, 1)$ such that (1) $\forall w_* \in W_*$, $\norm{Xw_*} \leq C_\tref{asp:X}$; (2) $\forall w, y$, $\norm{H(w, y)} \leq C_\tref{asp:X} (\norm{Xw} + 1)$; (3) For any $n \geq 1$:
            \begin{equation}
                \textstyle \norm{h(w) - \E[H(w, Y_{t+n}) | Y_t]} \leq C_\tref{asp:X}\tau^n \qty(\norm{Xw} + 1)
                \label{eq:uniform_mixing_in_A1}
            \end{equation}
\end{assumptionp}
This assumption is technically motivated but trivially holds in our analyses of~\eqref{eq:td lambda} and~\eqref{eq:artd}.
Specifically,
Assumption~\tref{asp:lipH} immediately leads to at-most-linear growth $\norm{H(w, y)} \leq C_{\tref{asp:lipH},1}(\norm{w} + 1)$ for some constant $C_{\tref{asp:lipH},1}$. 
However, this bound is insufficient for our analysis because $\norm{w} \leq \norm{w - \Gamma(w)} + \norm{\Gamma(w)}$ but $\Gamma(w) \in W_*$ can be unbounded.
By Assumption~\tref{asp:X},
we can have $\norm{Xw} \leq \norm{Xw - X\Gamma(w)} + \norm{X\Gamma(w)} \leq \norm{X} \norm{w - \Gamma(w)} + C_\tref{asp:X}$.
The inequality~\eqref{eq:uniform_mixing_in_A1} is related to geometrical mixing of the chain and we additionally include $Xw$ in the bound for the same reason.
We now present our general results regarding
the convergence rate of~\eqref{eq:Q_update_H_background} to $W_*$.
    \begin{theorem}
        \label{thm:sa}
        Let Assumptions~\tref{asp:lipH} - \tref{asp:X} and~\tref{assu lr} hold. Denote $\kappa \doteq \alpha C_\tref{asp:negdrift}$,
        then there exist some constants $t_0$ and $C_\text{Thm\ref{thm:sa}}$, such that the iterates $\qty{w_t}$ generated by~\eqref{eq:Q_update_H_background} satisfy for all $t$
        \begin{align*}
            \textstyle \E\qty[L(w_t)] \leq C_\text{Thm\ref{thm:sa},1}\qty(\frac{t_0}{t})^{\lfloor \kappa \rfloor}L(w_0) + C_\text{Thm\ref{thm:sa},2}\qty(\frac{\ln (t+t_0)}{(t+t_0)^{\min(2\xi-1, \lfloor \kappa \rfloor-1)}}).
        \end{align*}
    \end{theorem}
The proof is in Section~\tref{sec:thm_sa}. 
We remark that once we have the recursion in Lemma \ref{lem sa recur}, our theoretical framework can be readily extended to the constant step-size setting (akin to \citet{chen2023lyapunov}), demonstrating its broad applicability.
\section{Related Works}
Most prior works regarding the convergence of linear TD summarized in Table~\ref{tab:linear_td_comparison} rely on having linearly independent features.
In fact, the reliance on feature independence goes beyond linear TD and exists in almost all previous analyses of RL algorithms with linear function approximation, see, e.g., 
\citet{sutton2009convergent, sutton2009fast, maei2011gradient, hackman2012faster, liu2015gtd, yu2015convergence, yu2016weak, zou2019sarsa, yang2019provably, zhang2019provably,  xu2020improving,  zhang2020gradientdice, xu2020non, wu2020finite, chen2021finite, yang2021sarsa, qiu2021finite, zhang2020average, zhang2021breaking,  xu2021doubly, zhang2022globaloptimalityfinitesample, zhang2021truncated, zhang2023convergence, Chen2023global, fabbro2024finite, wang2024greedy, Swetha2024optimal, liu2024ode, qian2023revisit, maity2025adversarially, peng2025finite, chen2025nonasymptotic, haque2025stochastic, liu2025linearq}.
But as argued by \citet{dayan1992convergence,tsitsiklis1997analysis,tsitsiklis1999average,wang2024sureconvergencelineartemporal},
relaxing this assumption is an important research direction.
This work can be viewed as an extension of \citet{wang2024sureconvergencelineartemporal,zhang2021finite}.
In terms of~\eqref{eq:td lambda},
we extend \citet{wang2024sureconvergencelineartemporal} by proving a finite sample analysis.
Though we rely on the characterization of $W_*$ from \citet{wang2024sureconvergencelineartemporal},
the techniques we use for finite sample analysis are entirely different from the techniques of \citet{wang2024sureconvergencelineartemporal} for almost sure asymptotic convergence.
In terms of~\eqref{eq:artd},
we extend \citet{zhang2021finite} by allowing $X$ to be arbitrary.
Essentially, 
key to \citet{zhang2021finite} is their proof that $\overline A$ is n.d. on a subspace $E$, assuming $X$ has a full column rank.
We extend \citet{zhang2021finite} in that we give a finer and more detailed characterization of the counterparts of their $E$ through the novel decomposition of the features (Lemma~\tref{lem:z1z2}) and establish the n.d. property under weaker conditions (i.e., without assuming $X$ has a full column rank). 
% Crucially, our analysis demonstrates that this relaxation of the feature independence assumption does not compromise the convergence rate. Our derived rate of $\fO(\ln t/t^{2\xi-1})$ for a step-size of $\alpha_t \propto \frac{\alpha}{t^\xi}$ (where $\xi\in(0.5,1]$ and $\alpha$ is sufficiently large) is competitive with key benchmarks established under the linear independence assumption. For instance, \citet{bhandari2018finite} reported a $\fO(1/t)$ rate for a similar decaying step-size, while \citet{srikant2019finite} achieved $\fO(1/t)$ for constant step-sizes. For the common choice of $\xi=1$, our $\fO(\ln t/t)$ rate nearly matches these results and aligns with the standard convergence rate for Stochastic Gradient Descent (SGD), confirming that performance is maintained even with arbitrary features.
Importantly, despite relaxing the feature-independence assumption, our convergence rate remains on par with existing finite-sample results obtained under full-rank features \citep{bhandari2018finite,srikant2019finite}.
Our improvements are made possible by the novel Lyapunov function $L(w)$ and we argue that this Lyapunov function can be used to analyze many other linear RL algorithms with arbitrary features.

In terms of stochastic approximation,
our Theorem~\tref{thm:sa} is novel in that it allows convergence to a possibly unbounded set.
By contrast,
most prior works about stochastic approximation study convergence to a point 
\citep{borkar2000ode,chen2020finite,zhang2022globaloptimalityfinitesample,chen2023lyapunov,qian2024almost,liu2024ode, borkar2025ode,chen2025concentration}.
In the case of convergence to a set,
most prior works require the set to be bounded \citep{kushner2003stochastic,borkar2009stochastic,liu2024ode, liu2025extensions}.
Only a few prior works allow stochastic approximation to converge to an unbounded set, see, e.g., \citet{bravo2024stochastic,chen2025non,blaser2025asymptotic},
which apply to only tabular RL algorithms.

\section{Proofs of the Main Results}
\label{sec:proofs}
\subsection{Proof of Theorem~\tref{thm:sa}}
\label{sec:thm_sa}
\begin{proof}
From the 1-smoothness of $L(w)$ and \eqref{eq:Q_update_H_background}, we can get
% We first expand $L(w)$ from Assumption~\tref{asp:a3}. Using the 1-smoothness of $L(w)$ and \eqref{eq:Q_update_H_background}, we can get
\begin{align}
\label{eq L expand}
    L(w_{t+1}) &\textstyle \leq L(w_t)+\alpha_t\langle w_t-\Gamma(w_t),h(w_t)\rangle\notag \\
    &\textstyle\quad + \alpha_t\langle w_t-\Gamma(w_t),H(w_t,Y_t)-h(w_t)\rangle+
    \frac{1}{2}\alpha^2_t\norm{H(w_t,Y_t)}^2.
\end{align}
        We then bound the RHS one by one. $\langle w - \Gamma(w), h(w) \rangle$ is already bounded in Assumption~\tref{asp:negdrift}. 
    \begin{lemma}
    \label{lem wt bound}
        % Given Assumption~\tref{asp:W*}, 
        There exists a positive constant $C_\tref{lem wt bound}$, such that for any $w$,
        \begin{equation}
            \norm{Xw} \leq C_\tref{lem wt bound}(\norm{w-\Gamma(w)}+1).
        \end{equation}
    \end{lemma}
    The proof is in Section~\tref{proof wt bound}.
    With Lemma~\tref{lem wt bound} and Assumption~\tref{asp:X}, the last term in \eqref{eq L expand} can be bounded easily. 
    \begin{lemma}
    \label{lem bound H}
    There exists a constant $C_\tref{lem bound H}$ such that $\norm{H(w_t, Y_t)}^2 \leq C_\tref{lem bound H}(\norm{w_t - \Gamma(w_t)}^2 + 1)$.
    \end{lemma}
    The proof is in Section~\tref{proof bound H}.
    To bound $\langle w_t-\Gamma(w_t),H(w_t,Y_t)-h(w_t)\rangle$, leveraging \eqref{eq:uniform_mixing_in_A1}, we define 
    \begin{equation}
        \textstyle \tau_{\alpha} \doteq \min \{n \geq 0 \mid C_\tref{asp:X}\tau^n \leq \alpha\}
    \label{eq:tau_alpha}
    \end{equation}
    as the number of steps that the Markov chain needs to mix to an accuracy $\alpha$. In addition, we denote a shorthand $\alpha_{t_1, t_2} \doteq \sum_{i=t_1}^{t_2} \alpha_i$. 
    Then with techniques from \citet{srikant2019finite},
    we obtain
    % The main technique is to decompose this term using the mixing time $\tau_{\alpha}$ of the Markov chain, similarly as in \cite{zhang2022globaloptimalityfinitesample}.
    \begin{lemma}
    \label{lem:bound T}
    There exists a constant $C_\tref{lem:bound T}$ such that
    \begin{align*}
        \E\qty[\langle w_t-\Gamma(w_t), H(w_t,Y_t)-h(w_t)\rangle]\leq C_{\tref{lem:bound T}}\alpha_{t-\tau_{\alpha_t},t-1}(\norm{w_t-\Gamma(w_t)}^2 + 1). 
    \end{align*}
    \end{lemma}
    The proof is in Section \ref{proof:bound T}.
    Plugging all the bounds back to~\eqref{eq L expand}, we obtain 
     % Now, we bound the last term in the RHS.
    % Taking total expectation of \eqref{eq L expand} and plugging in the separate bounds, we get the following key recursive expression of the expected Lyapunov function $\E[L(w_t)]$.
    \begin{lemma}
    \label{lem sa recur}
        There exists some $D_t=\fO(\alpha_t\alpha_{t-\tau_{\alpha_t},t-1})$, such that 
        \begin{equation}
            \textstyle\E\qty[L(w_{t+1})] \leq (1 - C_\tref{asp:negdrift}\alpha_t) \E\qty[L(w_t)] + D_t.
        \end{equation}
    \end{lemma}
    The proof is in Section~\tref{proof sa recur}.
    Recursively applying Lemma~\tref{lem sa recur} then completes the proof of Theorem~\tref{thm:sa} (See Section~\tref{proof:sa} for details). 
\end{proof}
In the following sections, 
we first map the general update~\eqref{eq:Q_update_H_background} to~\eqref{eq:td lambda} and~\eqref{eq:artd} 
by defining 
$H(w,y)$, $h(w)$, and $L(w)$ properly.
Then we bound the remaining term $\langle \nabla L(w_t), h(w_t)\rangle$ to complete the proof. 

\subsection{Proof of Theorem~\tref{thm:td markov}}
\label{sec TD rate}
\begin{proof}
    We first rewrite \eqref{eq:td lambda} in the form of \eqref{eq:Q_update_H_background}.
% Recall the update rule in . To analyze its convergence, we reformulate it in the form of . Consider the stochastic process 
To this end, we define $Y_{t+1} \doteq (S_t, A_{t}, S_{t+1}, e_t)$, which evolves in an infinite space $\fY \doteq \fS \times \fA \times \fS \times \qty{e \mid \norm{e} \leq C_e }$
with $C_e \doteq \frac{\max_s \norm{x(s)}}{1 - \gamma\lambda}$ being the straightforward bound of $\sup_t \norm{e_t}$.
We define the incremental update $H: \R[d] \times \fY \to \R[d]$ as
\begin{align}
\label{eq:H def}
H(w, y) = (r(s, a) + \gamma x(s')^\top w - x(s)^\top w) e,
\end{align}
using shorthand $y = (s, a, s', e)$.
We now proceed to verifying the assumptions of Theorem~\tref{thm:sa}.
Assumption~\tref{asp:lipH} is verified by the following lemma.
\begin{lemma}
  \label{assu Lipschitz}
  There exists some finite $C_\tref{assu Lipschitz}$ such that
  \begin{align}
    \label{eq H Lipschitz}
        \norm{H(w_1,y)-H(w_2,y)} \leq& C_\tref{assu Lipschitz} \norm{w_1-w_2} \qq{$\forall w_1, w_2, y$.}
        % \norm{h(w_1)-h(w_2)} \leq& C_\tref{assu Lipschitz} \norm{w_1-w_2} \qq{$\forall w_1, w_2$.}
  \end{align}
  Moreover, $\sup_{y \in \fY} \norm{H(0, y)} < \infty$.
\end{lemma}
The proof is in Section~\tref{proof assu Lip}. \\
For Assumption~\tref{asp:mixing}, 
Theorem 3.2 of \citet{yu2012LSTD} confirms that $\qty{Y_t}$ has a unique stationary distribution $d_\fY$.
\citet{yu2012LSTD} also computes that
\begin{align}
    h(w)\doteq \E_{y\sim d_\mathcal{Y}}\qty[H(w, y)] = Aw + b.
\end{align}
% The geometric mixing~\eqref{eq:uniform_mixing_in_A1} is verified by Lemma~6.7 in \citet{bertsekas1996neuro}. 
Assumption~\tref{asp:W*} trivially holds by the definition of $W_*$. \\
% Denote $C_{\hat v_\pi}\doteq \norm{\hat v_\pi}$
% In particular, this also means that for any $w \in \R[d]$, we have $\norm{X\Gamma(w)} = C_{\hat v_\pi}$.
For Assumption~\tref{asp:negdrift},
the key observation is that 
$A \Gamma(w) + b =0$ always holds because $\Gamma(w) \in W_*$.
Then we have
$h(w) = A w + b = (A w + b)-(A \Gamma(w) + b) = A(w - \Gamma(w))$.
Thus the term $\inner{\nabla L(w)}{h(w)}$ can be written as $(w - \Gamma(w))^\top A(w - \Gamma(w))$.
We now prove that for whatever $X$, it always holds that $A$ is n.d. on $\ker(A)^\perp$.
\begin{lemma}
\label{prop:a}
There exists a constant $C_\tref{prop:a} > 0$ such that for $\forall w \in \ker(A)^\perp$, $w^\top A w \leq -C_\tref{prop:a}\norm{w}^2$.
Furthermore, for any $w \in \R[d]$, it holds that $w - \Gamma(w) \in \ker(A)^\perp$.
\end{lemma}
    The proof is in Section~\tref{proof:a}.
    We then have
    % Applying this lemma with $z = w_t - \Gamma(w_t)$, we have:
\begin{align*}
    \langle w_t - \Gamma(w_t), A(w_t - \Gamma(w_t))\rangle \leq -C_\tref{prop:a}\norm{w_t - \Gamma(w_t)}^2,
\end{align*}
which satisfies Assumption~\tref{asp:negdrift}. \\
For Assumption~\tref{asp:X},~\eqref{eq limiting value} verifies Assumption~\tref{asp:X}(1).
Assumption~\tref{asp:X}(2) is verified by the following lemma.
\begin{lemma}
\label{lem H h linear growth}
    There exists a constant $C_\tref{lem H h linear growth}$ such that for $\forall w, y$, $\norm{H(w, y)} \leq C_\tref{lem H h linear growth} (\norm{Xw} + 1)$.
\end{lemma}
The proof is in Section~\tref{proof H h linear growth}. Assumption~\tref{asp:X}(3) is verified following a similar procedure as Lemma~6.7 in \citet{bertsekas1996neuro} (Lemma~\tref{lem:td_mix}).
Invoking Theorem~\tref{thm:sa} then completes the proof.
% \begin{align}
%\label{eq:L expand}
%    \E\qty[L(w_{t+1})] \leq \E\qty[L(w_t)] -2C_\tref{prop:a} \alpha_t L(w_t) + \alpha_t (\E\qty[T_1] + \E\qty[T_2] + \E\qty[T_3]) + \frac{1}{2} \alpha_t^2 \norm{H(w_t, Y_t)}^2.\quad
% \end{align}
% After integration, we get the following lemma.
% \begin{lemma}
% \label{lem:integrate}
% There exists some $D_t=\fO(\alpha_t\alpha_{t-\tau_{\alpha_t},t-1})$, such that
%\begin{align*}
%     \E\qty[L(w_{t+1})] \leq (1 - C_\tref{prop:a}\alpha_t) \E\qty[L(w_t)] + D_t.
% \end{align*}
% \end{lemma}
%    The proof is in Section \ref{proof:integrate}. Apply this Lemma~\ref{lem:integrate} inductively, we then finish the proof of Theorem~\ref{thm:td markov} with constant $\kappa \doteq C_\tref{prop:a} \alpha$ (see Section \ref{proof:td markov}).
\end{proof}

\subsection{Proof of Theorem~\tref{thm:ar td markov}}
\label{sec:ar}
\begin{proof}
We recall that in view of Lemma~\tref{lem:z1z2}, $\ker(X_1)$ creates ``duplication'' in value estimation.
We, therefore, define the projection matrix $\Pi \in \R[d \times d]$ that projects a vector into the orthogonal complement of $\ker(X_1)$, i.e.,
$\Pi w \doteq \arg\min_{w' \in \ker(X_1)^\perp} \norm{w - w'}$.
It can be computed that $\Pi = X_1^\dagger X_1$.
% We recall that the subspace $E$ defined in Section~\ref{thms} is the ``effective'' part of $\R[d]$ and $\Pi_{2, E}(w)$ captures the ``effective'' part of $w$ (in terms of value estimation).
% Since $\Pi_{2, E}$ is an orthogonal projection to the subspace $E$,
% the function $\Pi_{2, E}(w)$ is essentially to multiply the corresponding projection matrix over $w$.
% In view of this, 
% in the rest of the proof we understand $\Pi_{2, E}$ as a $d \times d$ matrix.
% Following \citet{zhang2021finite},
We now examine the sequence $\qty{\Pi w_t}$ with $\qty{w_t}$ being the iterates of~\eqref{eq:artd}
and consider the combined parameter vector 
$\tw_t \doteq \textstyle\begin{bmatrix}
    \hat{J}_t\\ \Pi w_t
\end{bmatrix} \in \R[1+d]$.
The following lemma characterizes the evolution of $\tw_t$.
Let $Y_t = (S_t, A_t, S_{t+1}, e_t) \in \fS \times \fA \times \fS \times \qty{e \in \R[d] \mid \norm{e} \leq \frac{\max_s \norm{x(s)}}{1 - \lambda}}$, then
\begin{lemma}
\label{lem:9}
        $\tw_{t+1} = \tw_t + \alpha_t (\widetilde A(Y_t) \tw_t + \widetilde b(Y_t))$,
    where we have, with $y = (s, a, s', e)$,
\begin{align}
\label{eq:def Ab ar}
        \widetilde A(y) = \begin{bmatrix} 
        -c_\beta & \0 \\
        -\Pi e & \Pi e(x(s')^\top - x(s)^\top)
        \end{bmatrix}, \widetilde b(y) = \begin{bmatrix}
        c_\beta r(s, a) \\
        r(s, a) \Pi e
    \end{bmatrix}.
\end{align}
\end{lemma}
This view is inspired by \citet{zhang2021finite} and the proof is in Section~\tref{proof:9}.
We now apply Theorem~\tref{thm:sa} to $\qty{\tw_t}$.

The verification of Assumptions~\tref{asp:lipH} and~\tref{asp:mixing} is identical to that in Section~\tref{sec TD rate} and is thus omitted.
For Assumption~\tref{asp:W*},
we define $\widetilde W_* \doteq \left\{\begin{bmatrix} J_\pi \\ \Pi  w \end{bmatrix}\middle| w \in \overline W_*\right\}$.
It is apparently nonempty, closed, and convex. \\
For Assumption~\tref{asp:negdrift},
we define $\widetilde A \doteq \E_{y\sim d_\mathcal{Y}} \qty[\widetilde A(y)]$ and $\widetilde b \doteq \E_{y\sim d_\mathcal{Y}} \qty[\widetilde b(y)]$ and therefore realize the $h$ in~\eqref{eq:Q_update_H_background} as $h(\tw) = \widetilde A \tw + \widetilde b$.
% Since $\Gamma(\tw) \in \widetilde W_*$ by definition of the projection, 
Noticing that $\widetilde A \Gamma(\tw) + \widetilde b = \0$ (Lemma~\tref{lem:gamma ar}),
we then have $h(\widetilde w) = \widetilde A(\widetilde w - \Gamma(\widetilde w))$.
The term $\inner{\nabla L(\widetilde w)}{h(\widetilde w)}$ can thus be written as $(\widetilde w - \Gamma(\widetilde w))^\top \widetilde A(\widetilde w - \Gamma(\widetilde w))$.
% Since $\Gamma(\tw) \in \widetilde W_*$ by definition of the projection, it holds that $\widetilde A \Gamma(\tw) + \widetilde b = \0$. 
% Here, for the first row, we get $c_\beta J_\pi-J_\pi = 0$. 
% For the second row, we get $\Pi(\overline Aw+\overline b)$ for $w \in \overline W_*$, which gives us $\Pi(0)=0$ according to the definition of $\Pi$.
% Then similarly as in the discounted case, we get $h(\widetilde w) = \widetilde A(\widetilde w - \Gamma(\widetilde w))$, and the term $\inner{\nabla L(\widetilde w)}{h(\widetilde w)}$ can be written as $(\widetilde w - \Gamma(\widetilde w))^\top \widetilde A(\widetilde w - \Gamma(\widetilde w))$.
Next, we prove that when $c_\beta$ is large enough, $\widetilde A$ is n.d. on $\R \times \ker(X_1)^\perp$.
\begin{lemma}
\label{lem negdef pA}
Let $c_\beta$ be sufficiently large. Then there exists a constant $C_\tref{lem negdef pA}>0$ such that $\forall z \in \R \times \ker(X_1)^\perp$, $z^T \pA z \leq -C_\tref{lem negdef pA} \norm{z}^2$.
% Furthermore, for any $\tw \in \R[1+d]$, it holds that $\widetilde w - \Gamma(\widetilde w) \in \R \times \ker(X_1)^\perp$.
\end{lemma}
The proof is in Section~\ref{proof:negdef pA}. 
By definition, we have $\tw_t \in \R \times \ker(X_1)^\perp$ and $\Gamma(\tw)\in \R \times \ker(X_1)^\perp$. So $\tw - \Gamma(\tw) \in \R \times \ker(X_1)^\perp$, yielding
\begin{align*}
    \langle \widetilde w_t - \Gamma(\widetilde w_t), \widetilde A(\widetilde w_t - \Gamma(\widetilde w_t))\rangle \leq -C_\tref{lem negdef pA}\norm{\widetilde w_t - \Gamma(\widetilde w_t)}^2,
\end{align*}
which verifies Assumption~\tref{asp:negdrift}. \\
For Assumption~\tref{asp:X},
we define $\widetilde{X}=\mqty[ 1 & \mathbf{0}^\top \\ \mathbf{0} & X]$. 
Assumption~\tref{asp:X}(1) is verified below.
\begin{lemma}
\label{lem gamma bound ar}
There exists a positive constant $C_\tref{lem gamma bound ar}$, such that for any $\tw \in \widetilde W_*$, $\textstyle \norm{\widetilde{X}\tw} = C_\tref{lem gamma bound ar}$.
\end{lemma}
The proof is in Section~\ref{proof gamma bound ar}.
With $H(\tw,y) = \widetilde A(y)\tw+\widetilde b(y)$,
the verification of Assumption~\tref{asp:X}(2) and (3) is similar to Lemmas~\tref{lem H h linear growth} and \tref{lem:td_mix} and is thus omitted.
Invoking Theorem~\ref{thm:sa} then yields the convergence rate of $\E \qty[L(\tw_t)]$,
i.e.,
the convergence rate of $d(\tw_t,\widetilde{W}_*)^2$ by the definition of $L$.
The next key observation is that 
$d(\tw_t, \widetilde W_*)^2 = (\hat{J}_t-J_\pi)^2+d(w_t, \overline W_*)^2$ (Lemma~\ref{lem:dd}),
which completes the proof. 
% $\Pi w$ is always orthogonal to $\overline W_*$.
% \begin{lemma}
% \label{lem:ar}
% Under the same conditions as Theorem~\ref{thm:ar td markov}, with the addition that $\lambda < 1$, the auxiliary algorithm of average reward TD($\lambda$) achieves
% \begin{align*}
%     \textstyle \E\qty[L(\tw_t)] \leq \qty(\frac{t_0}{t})^{\lfloor \kappa \rfloor}L(\tw_0) + C_\tref{lem:ar}\qty(\frac{\ln (t+t_0)}{(t+t_0)^{2\xi-1}}).
% \end{align*}
% \end{lemma}
% By definition, the projection direction of $\Pi$ is parallel to $W_*$, 
% we derive that . In addition, notice that $L(\tw_t)=\frac{1}{2}d(\tw_t,\widetilde{W}_*)^2$, we then complete the proof of Theorem~\ref{thm:ar td markov}. 
% \xl{I revised this paragraph. Check whether it works.}\zx{Great!}
\end{proof}

\section{Experiments}
\label{sec:experiment}
We now empirically examine linear TD with linearly dependent features.
Following the practice of \citet{sutton2018reinforcement}, 
we use diminishing learning rates $\alpha_t = \frac{\alpha}{t+t_0}$ and $\beta_t= \frac{\beta}{t+t_0}$ which closely match our Assumption~\tref{assu lr} with $\xi =1$ and $t_0 = 10^7$.
We use a variant of Boyan's chain \citep{boyan1999least} with 15 states ($|\mathcal{S}|=15$) and 5 actions ($|\mathcal{A}|=5$) under a uniform policy $\pi(a|s) = 1/|\mathcal{A}|$, where the feature matrix $X \in \mathbb{R}^{15 \times 5}$ is designed to be of rank 3 (more details in Section~\ref{sec:detail experiments}).\footnote{The code for this paper is available at \url{https://github.com/WennyXie/LinearTDLambda}.} 
The weight convergence to a set is indeed observed.
It is within expectation that different $\lambda$ requires different $\alpha, \beta$. 
% Experiments cover both discounted and average-reward settings, focusing on the impact of parameters like $\gamma$, $\lambda$, and $\alpha_0$.
\begin{figure}[h]
    \centering
    \includegraphics[width=1.0\textwidth]{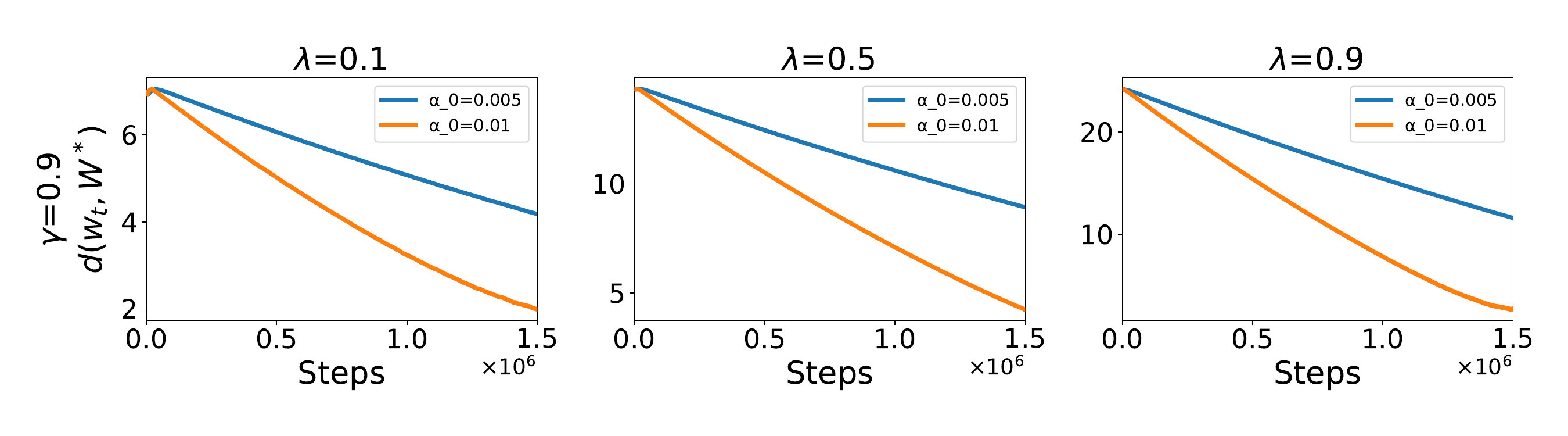}
    \caption{Convergence of~\eqref{eq:td lambda} with $\gamma = 0.9, \alpha_0 \in \qty{0.005, 0.01}$.
    Curves are averaged over 10 runs with shaded regions (too small to be visible) indicating standard errors.
    % Linear TD($\lambda$) in the discounted setting on Boyan's chain, measured as the projection distance $d(w_t, W_*)$. The subplot corresponds to a $(\gamma, \lambda)$ pair with $\alpha \in \{0.005, 0.01\}$ when $\gamma \in \{0.9\}$. 
    % Shaded regions show one standard deviation over 10 runs.
    }
    \label{fig:convergence_discounted}
\end{figure}
\begin{figure}[h]
    \centering
    \includegraphics[width=1.0\textwidth]{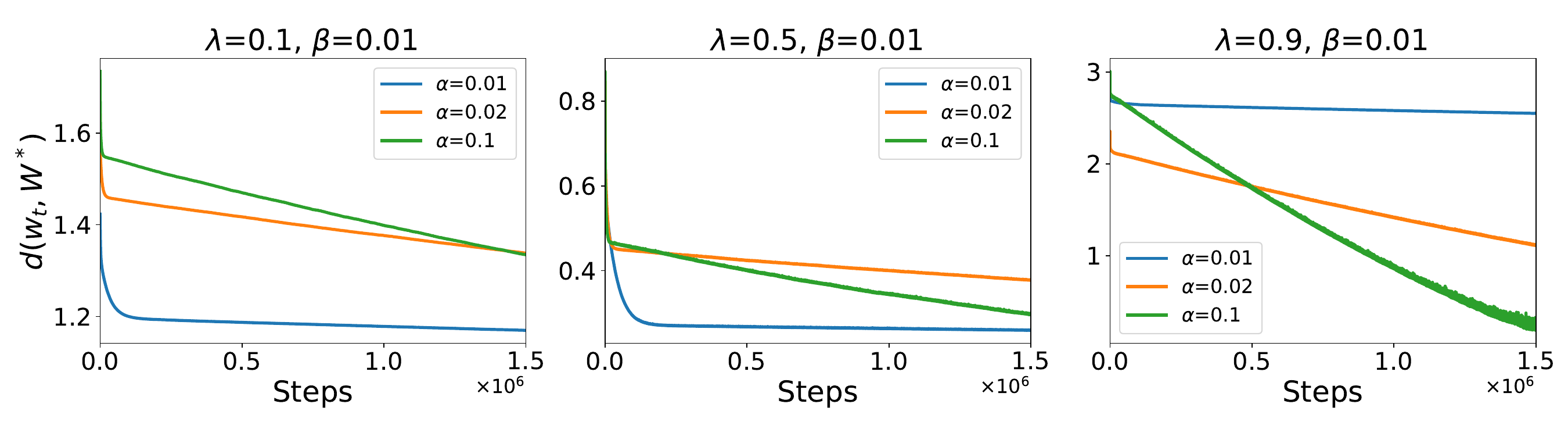}
    \caption{Convergence of~\eqref{eq:artd} with $\beta_0 = 0.01, \alpha_0 \in \qty{0.01, 0.02, 0.1}$.
    Curves are averaged over 10 runs with shaded regions (too small to be visible) indicating standard errors.
    %  in the average-reward setting on Boyan's chain, measured as the projection distance $d(w_t, W_*)$. Each subplot corresponds to a $(c_\delta, \lambda)$ pair with $\beta = 0.01$.
     }
    \label{fig:convergence_average_reward}
\end{figure}

\section{Conclusion}
This paper provides the first finite sample analysis of linear TD with arbitrary features in both discounted and average reward settings,
fulfilling the long standing desiderata of \citet{dayan1992convergence,tsitsiklis1997analysis,tsitsiklis1999average},
enabled by a novel stochastic approximation result concerning the convergence rate to a set.
The key methodology contributions include a novel Lyapunov function based on the distance to a set and a novel decomposition of the feature matrix for the average-reward setting.
We envision the techniques developed in this work can easily transfer to the analyses of other linear RL algorithms.
That being said,
one limitation of the work is its focus on linear function approximation.
Extension to neural networks with neural tangent kernels (cf. \citet{cai2023neural}) is a possible future work.
Another limitation is that this work considers only $L^2$ convergence rates but the convergence mode of random variables are versatile.
Establishing almost sure convergence rates, $L^p$ convergence rates, and high probability concentration bounds (cf. \citet{qian2024almost}) is also a possible future work. Finally, another promising direction is the integration of Polyak-Ruppert averaging (cf. \citet{patil2023finite, naskar2024federated}), which potentially leads to parameter-free convergence rates.

\section*{Acknowledgments and Disclosure of Funding}
This work is supported in part by the US National Science Foundation under the awards III-2128019, SLES-2331904, and CAREER-2442098, the Commonwealth Cyber Initiative's Central Virginia Node under the award VV-1Q26-001, and an Nvidia academic grant program award.
{\small
\bibliography{bibliography}
}

\clearpage
\appendix
\section{Auxiliary Lemmas and Notations}
\begin{lemma}
  [Discrete Gronwall Inequality, Lemma 8 in Section 11.2 of \citet{borkar2009stochastic}]
\label{lem discrete_gronwall}
    For non-negative real sequences $\qty{x_n, n \geq 0}$ and $\qty{a_n, n \geq 0}$ and scalar  $L \geq 0$,  it holds
    \begin{align}
        \textstyle x_{n+1} \leq C + L\sum_{i=0}^n a_ix_i \quad \forall n
    \implies
        \textstyle x_{n+1} \leq (C+x_0)\exp({L\sum_{i=0}^n a_i}) \quad \forall n.
    \end{align}
\end{lemma}
\begin{lemma}[Lemma 11 of \citet{zhang2022globaloptimalityfinitesample}]
\label{bound_tau_alpha}
For sufficiently large $t_0$, it holds that
    \begin{equation}
        \tau_{\alpha_t} = \fO\qty( \log (t+t_0)), \quad \alpha_{t-\tau_{\alpha_t}, t-1} = \fO\qty(\frac{\log (t+t_0)}{(t+t_0)^{\xi}}).
    \end{equation}
\end{lemma}
Lemma~\ref{bound_tau_alpha} ensures that there exists some $\overline t > 0$ (depending on $t_0$) such that for all $t \geq \overline t$,
it holds that $t \geq \tau_{\alpha_t}$.
Also, it ensures that for sufficiently large $t_0$, we have $\alpha_{t-\tau_{\alpha_t}, t-1}<1$.
Throughout the appendix, 
we always assume $t_0$ is sufficiently large and $t \geq \overline t$.
We will refine (i.e., increase) $\overline t$ along the proof when necessary.

\section{Proofs in Section~\tref{thms}}
\label{sec proof main}
\subsection{Proof of Lemma~\tref{lem:z1z2}}
\label{proof:z1z2}
\begin{proof}
    Let $x_i \in \R[d]$ denote the $i$-th column of $X$.
    Without loss of generity, let the first $m$ columns be linearly independent. \\
    \textbf{Case 1:} When $\1 \in \col(X)$, there must exist $m$ scalars $\qty{c_i}$ such that
    $\sum_{i=1}^m c_i x_i = \1$. 
    Apparently, at least one of $\qty{c_i}$ must be nonzero.
    Without loss of generity, let $x_m \neq 0$.
    We then have
    \begin{align}
        x_m = \frac{1}{c_m}(\1 - \sum_{i=1}^{m-1} c_i x_i).
    \end{align}
    In other words, $x_m$ can be expressed as linear combination of $\qty{x_1, \dots, x_{m-1}}$ and $\1$.
    Since $X$ has a column rank $m$,
    we are able to express $\qty{x_{m+1}, \dots, x_d}$ by linear combination $\qty{x_1, \dots, x_m}$ and thus further by linear combination of $\qty{x_1, \dots, x_{m-1}}$ and $\1$.
    Let $Z_1 \doteq \mqty[x_1, \dots, x_{m-1}]$ be the first $m-1$ columns of $X$ and $Z_2 \doteq \mqty[x_m, \dots, x_d]$ be the rest.
    We now know that there exists some $C \in \R[(m-1) \times (d-m+1)]$ (i.e., coefficients of the lienar combination) such that
    \begin{align}
        Z_2 = Z_1 C + \mqty[\theta_m \1, \dots, \theta_d \1],
    \end{align}
    where $\theta_m, \dots \theta_d$ are scalars (i.e., ``coordinates'' along the $\1$-axis), e.g., $\theta_m = \frac{1}{c_m}$.
    This means that we can express $X$ as
    \begin{align}
    \label{eq:defx1x2}
        X = \mqty[Z_1 & Z_1C] + \mqty[\theta_1 \1, \dots, \theta_d \1]
    \end{align}
    with $\theta_1 = \dots = \theta_{m-1} = 0$.
    Now define
    \begin{align}
        X_1 \doteq \mqty[Z_1 & Z_1C], \, X_2 \doteq \mqty[\theta_1 \1, \dots, \theta_d \1].
    \end{align}
    % and $Z_3\doteq \mqty[\theta_m \1, \dots, \theta_d \1]$.
    We note that $\1 \notin \col(Z_1)$.
    Otherwise, there would exist scalars $\qty{c_i'}$ such that $\sum_{i=1}^{m-1} c_i' x_i = \1$.
    Then we get $\sum_{i=1}^{m-1} (c_i - c_i')x_i + c_m x_m = 0$,
    which is impossible because $\qty{x_i}_{i=1,\cdots, m}$ are linearly independent.
    % Define
    % \begin{align}
    %     X_1 \doteq \mqty[Z_1 & Z_1C], X_2 \doteq \mqty[\theta_1 \1, \dots, \theta_d \1].
    % \end{align}
    Since $\col(X_1) = \col(Z_1)$, we then have $\1 \notin \col(X_1)$. \\
    \textbf{Case 2:} When $\1 \notin \col(X)$, we can trivially define $X_1 = X$ and $X_2 = 0$.
    Additionally,
    we can still further decompose $X_1$ as 
    \begin{align}
    \label{eq:defx1x2 2nd}
      X_1 = \mqty[Z_1 & Z_1C],
    \end{align}
    where $Z_1$ is now the first $m$ columns of $X$.
    Apparently, we still have $\1 \notin \col(X_1)$.
    % which then completes the proof.
    % \sz{Discuss the case when $\1 \notin \col(X)$. Move this proof to appendix.}
    % Now, briefly consider the case when $\1 \notin \col(X)$. In this scenario, the constant feature direction is absent from the feature space, simplifying the analysis. The solution structure is typically $W^* = \{w^*\} + \ker(X)$, and the specific decomposition constructed above is not required. \zx{Not sure about this case...}
    % Let $Z$ be the first $m-1$ columns of $X$ and define $c = [c_1, \dots, c_{m-1}]^\top$.
    % We then have 
    % \begin{align}
    %     x_m = \frac{1}{m}(\1 - Z c).
    % \end{align}
    % Since $X$ has a column rank $m$,
    % we are able to express the columns $\qty{x_{m+1}, \dots, x_d}$ as linear combination of $\qty{x_1, \dots, x_m}$.
    % Let $Z \in \R[\ns \times d]$ be the first $m-1$ columns of $X$.
    % It is easy to see that $\1 \notin \col(Z)$.
\end{proof}

\begin{lemma}
    \label{lem:A bar}
Let Assumption~\ref{assum:markov} hold. Then
    $\overline A = X_1 D_\pi (P_\lambda - I)X_1, \overline b = X_1^\top D_\pi (r_\lambda - \frac{J_\pi}{1 - \lambda} \1)$.
\end{lemma}
\begin{proof}
Apply the decomposition shown in Lemma~\ref{lem:z1z2}, we can get
\begin{align}
    \overline A
    =& (X_1+X_2)^\top D_\pi(P_\lambda-I)(X_1+X_2)\\
    =&X_1^\top D_\pi(P_\lambda-I)X_1+X_2^\top D_\pi(P_\lambda-I)X_1+ X_1^\top D_\pi(P_\lambda-I)X_2+X_2^\top D_\pi(P_\lambda-I)X_2\\
    =&X_1^\top D_\pi(P_\lambda-I)X_1,
\end{align}
where the last equality holds because 
$(P_\lambda-I)\1=0$ and $\1^\top D_\pi(P_\lambda-I)=d_\pi^\top P_\lambda - d_\pi^\top=0$. 
Similarly, for $\overline b$ we can obtain
\begin{align}
    \overline b
    =& (X_1+X_2)^\top D_\pi (r_\lambda - \frac{J_\pi}{1 - \lambda} \1)\\
    =& X_1^\top D_\pi (r_\lambda - \frac{J_\pi}{1 - \lambda} \1) + X_2^\top D_\pi (r_\lambda - \frac{J_\pi}{1 - \lambda} \1)\\
    =& X_1^\top D_\pi (r_\lambda - \frac{J_\pi}{1 - \lambda} \1) + \theta^\top (d_\pi^\top (I-\lambda P_\pi)^{-1}r_\pi - \frac{J_\pi}{1 - \lambda})\\
    =& X_1^\top D_\pi (r_\lambda - \frac{J_\pi}{1 - \lambda} \1) + \theta^\top (\frac{1}{1-\lambda}d_\pi^\top r_\pi - \frac{J_\pi}{1 - \lambda})\\
    =& X_1^\top D_\pi (r_\lambda - \frac{J_\pi}{1 - \lambda} \1).
\end{align}
Here, the fourth inequality holds because $d_\pi^\top(I-\lambda P_\pi) = (1 - \lambda) d_\pi^\top$, which gives us $d_\pi^\top = (1 - \lambda) d_\pi^\top (I-\lambda P_\pi)^{-1}$. The last inequality holds since $J_\pi = d_\pi^\top r_\pi$. This completes the proof.
\end{proof}

\begin{lemma}
\label{lem:w*}
    Let Assumption~\ref{assum:markov} hold. Then $\overline W_*$ is nonempty.
\end{lemma}
\begin{proof}
    In view of~\eqref{eq:defx1x2} and~\eqref{eq:defx1x2 2nd},
we have $X_1 = \mqty[Z_1 & Z_1C]$.
Notably, $Z_1$ has a full column rank and $\1 \notin \col(Z_1)$.
Decompose $w \doteq \begin{bmatrix}
    w_1 \\ w_2
\end{bmatrix}$ accordingly and recall \eqref{eq:bar W*} and Lemma~\tref{lem:A bar}, we can rewrite $\overline Aw + \overline b = 0$ as
\begin{align}
    \begin{bmatrix}
    Z_1^\top \\ (Z_1C)^\top
\end{bmatrix}D_\pi(P_\lambda-I)\mqty[Z_1 & Z_1C]\begin{bmatrix}
    w_1 \\ w_2
\end{bmatrix} = \begin{bmatrix}
    -Z_1^\top D_\pi (r_\lambda - \frac{J_\pi}{1 - \lambda} \1) \\ -(Z_1C)^\top D_\pi (r_\lambda - \frac{J_\pi}{1 - \lambda} \1)
\end{bmatrix},
\end{align}
which thus gives us the following simultaneous equations
\begin{equation}
    \begin{cases}
    Z_1^\top D_\pi(P_\lambda-I)Z_1w_1+Z_1^\top D_\pi(P_\lambda-I) Z_1Cw_2 = -Z_1^\top D_\pi (r_\lambda - \frac{J_\pi}{1 - \lambda} \1)\\
    (Z_1C)^\top D_\pi(P_\lambda-I)Z_1w_1+(Z_1C)^\top D_\pi(P_\lambda-I) Z_1Cw_2 = -(Z_1C)^\top D_\pi (r_\lambda - \frac{J_\pi}{1 - \lambda} \1)
\end{cases}.
\end{equation}
We now prove the claim by constructing a solution.
Choose any $w_2\in \ker(Z_1C)$ (e.g., $w_2 = 0$), the equations then become
\begin{align}
\begin{cases}
    Z_1^\top D_\pi(P_\lambda-I)Z_1w_1=-Z_1^\top D_\pi (r_\lambda - \frac{J_\pi}{1 - \lambda} \1)\\
    C^\top Z_1^\top D_\pi(P_\lambda-I)Z_1w_1=-C^\top Z_1^\top D_\pi (r_\lambda - \frac{J_\pi}{1 - \lambda} \1).
\end{cases} 
\end{align}
Since $Z_1$ is full rank and $\1 \notin Z_1$, 
Lemma~7 of \citet{tsitsiklis1999average} shows $Z_1^\top D_\pi(P_\lambda-I)Z_1$ is n.d. and thus invertible. 
Choose $w_1 =- (Z_1^\top D_\pi(P_\lambda-I)Z_1)^{-1}Z_1^\top D_\pi (r_\lambda - \frac{J_\pi}{1 - \lambda} \1)$ then satisfies the equations.
This completes the proof.

\end{proof}

\begin{lemma}
\label{lem:fix_points}
    Let Assumption~\ref{assum:markov} hold. Then \\$\overline W_*= \{\overline w_*\} + \ker(X_1)$ and $\ker(X_1) = \qty{w | Xw = c\1, c \in \R}$.
% \begin{align*}
%     W_*= \{w_*\} + \ker(X_1).
% \end{align*}
\end{lemma}
\begin{proof}
    For any solution $w_*, w_{**} \in \overline W_*$, according to the definition of $\overline W_*$ in \eqref{eq:bar W*}, we have $\overline A w_*+ \overline b=\0$ and $\overline A w_{**}+ \overline b=\0$. 
    That is $\overline A(w_*- w_{**})= \0$. By multiplying $(w_*- w_{**})^\top$ on both side we can get
    \begin{align}
        (w_*-  w_{**})^\top X^\top D_\pi(P_\lambda-I)X(w_*- w_{**})=0.
    \end{align}
    According to the Perron-Frobenius theorem with Assumption~\ref{assum:markov}, $v^\top D_\pi(P_\lambda-I) v = 0$ if and only if $v = c\1$ for some $c \in \R$. 
    Therefore, we must have $X( w_*-w_{**})=c\1$ for some $c\in \R$. That is, $(X_1 + X_2)( w_*- w_{**})=c\1$. 
    Recall the definition of $X_2$ in \eqref{eq:defx1x2}, 
    we have $X_2(w_* - w_{**}) = (\theta^\top (w_* - w_{**})) \1$.
    % $X_2w$ is always a multiple of $\1$, let $X_2w = c'\1$ for some $c' \in \R$. 
    This means
    $X_1(w_*- w_{**})=c'\1$ with $c' = c - \theta^\top ( w_* - w_{**})$. 
    Since $\1 \notin \col(X_1)$, we must have $c'=0$. 
    That is, $w_*- w_{**} \in \ker(X_1)$.
    Thus, we have established that $\overline W_* = \qty{\overline w_*} + \ker(X_1)$.
    
    Furthermore, 
    if $w \in \ker(X_1)$,
    we have $Xw = (X_1 + X_2)w = (\theta^\top w)\1$.
    If $Xw = c \1$,
    we have $X_1w = c \1 - X_2 w = (c - \theta^\top w) \1$.
    But $\1 \notin \col(X_1)$.
    So we must have $c - \theta^\top w = 0$, i.e., $w \in \ker(X_1)$.
    This completes the proof of $\ker(X_1) = \qty{w | Xw = c\1, c \in \R}$.
\end{proof}

\section{Proofs in Section~\tref{sec:thm_sa}}
\label{sec proof sa}
% \subsection{Proof of Lemma~\tref{lem bound T}}
% \label{proof bound T}
% \begin{proof}
    
% \end{proof}
% \sz{As in the ICML paper, you need to introduce $\overline t$ everywhere.}
\begin{lemma}
    \label{bound_wt_wt1}
    For sufficiently large $t_0$,
    there exists a constant $C_\tref{bound_wt_wt1}$ such that the following statement holds.
    For any $t \geq \overline t$ and any $i \in [t - \tau_{\alpha_t}, t]$,
    it holds that
        % given positive integers $t_1 < t_2$, for any $t \in [t_1,t_2]$,
        % it holds that,
        \begin{align}
            & \norm{w_i - w_{t-\tau_{\alpha_t}}} \leq C_\tref{bound_wt_wt1}\alpha_{t-\tau_{\alpha_t},i-1}(\norm{w_i-\Gamma(w_i)} + 1)\label{eq:1}.
        \end{align}
    \end{lemma}
    \begin{proof}
        In this proof, to simplify notations,
        we define shorthand $t_1 \doteq t - \tau_{\alpha_t}$ and $C_x \doteq \max_{s} \norm{x(s)}$.
        Given Lemma~\ref{bound_tau_alpha},
        we can select a sufficiently large $t_0$ such that for any $t \geq \overline t$,
        \begin{align}
\exp(C_\tref{asp:X}C_x \alpha_{t- \tau_{\alpha_t} t - 1}) <& 3, \\
C_\tref{asp:X}C_x \alpha_{t- \tau_{\alpha_t} t - 1} <& \frac{1}{6}.
        \end{align}
    %    Define $C_x \doteq \max_{s} \norm{x(s)}$, where $x(s)^\top$ is the $s$-th row of $X$. 
       We then bound $\norm{w_i - w_{t_1}}$ as 
        \begin{align}
            \norm{w_i - w_{t_1}} &\leq \sum_{k=t_1}^{i-1}\norm{\alpha_k H(w_k, Y_{k+1})}\\
            &\leq \sum_{k=t_1}^{i-1}\alpha_k C_\tref{asp:X}(\norm{Xw_k-Xw_{t_1}}+\norm{Xw_{t_1}}+1)\quad\text{(Assumption~\tref{asp:X})}\\
            &\leq \sum_{k=t_1}^{i-1}\alpha_k C_\tref{asp:X}(\norm{Xw_{t_1}}+1) + \sum_{k=t_1}^{i-1}\alpha_k C_\tref{asp:X}(\norm{Xw_k-Xw_{t_1}})\\
            &\leq \sum_{k=t_1}^{i-1}\alpha_k C_\tref{asp:X}(\norm{Xw_{t_1}}+1) + \sum_{k=t_1}^{i-1}\alpha_k C_{\tref{bound_wt_wt1},1}(\norm{w_k-w_{t_1}})\\
            &\leq C_\tref{asp:X}\alpha_{t_1,i-1} (\norm{Xw_{t_1}} + 1)\exp(C_{\tref{bound_wt_wt1},1}\alpha_{t_1, t - 1}), \qq{(Lemma~\ref{lem discrete_gronwall})}
        \end{align}
        where $C_{\tref{bound_wt_wt1},1}\doteq C_\tref{asp:X}C_x$.
        We then have
        \begin{align}
            &\norm{w_i - w_{t_1}}\\ 
            \leq& C_\tref{asp:X}\alpha_{t_1,i-1} (\norm{Xw_i - Xw_{t_1}} + \norm{Xw_i} + 1)\exp(C_{\tref{bound_wt_wt1},1}\alpha_{t_1, t - 1}) \\
            \leq& C_\tref{asp:X}C_x\exp(C_{\tref{bound_wt_wt1},1}\alpha_{t_1, t - 1})\alpha_{t_1,i-1} \norm{w_i - w_{t_1}} + \exp(C_{\tref{bound_wt_wt1},1}\alpha_{t_1, t - 1})(\norm{Xw_i} + 1)C_\tref{asp:X}\alpha_{t_1,i-1} \\
            \leq& \frac{1}{2}\norm{w_{i}-w_{t_1}} + C_{\tref{bound_wt_wt1},2}\alpha_{t_1,i-1}(\norm{Xw_{i}} + 1),
        \end{align}
        where $C_{\tref{bound_wt_wt1},2} \doteq 3 C_\tref{asp:X}$.
        Thus, we have
        \begin{align}
            \norm{w_{i} - w_{t_1}} &\leq 2C_{\tref{bound_wt_wt1},2}\alpha_{t_1,i-1}(\norm{Xw_{i}} + 1)\\
            &\leq 2C_{\tref{bound_wt_wt1},2}\alpha_{t_1,i-1}(C_\tref{lem wt bound}(\norm{w_i-\Gamma(w_i)}+1)+1)\\
            &\leq C_\tref{bound_wt_wt1}\alpha_{t_1,i-1}(\norm{w_i-\Gamma(w_i)} + 1),
        \end{align}
        where $C_\tref{bound_wt_wt1} \doteq 2C_{\tref{bound_wt_wt1},2}(C_\tref{lem wt bound}+1)$.
        This completes the proof.
    \end{proof}

\subsection{Proof of Lemma~\tref{lem wt bound}}
\label{proof wt bound}
\begin{proof}
    \begin{align}
        \norm{Xw} &= \norm{Xw-X\Gamma(w)+X\Gamma(w)}\\
        &\leq \norm{X(w-\Gamma(w))}+\norm{X\Gamma(w)}\\
        &\leq \norm{X}\norm{w-\Gamma(w)}+C_\tref{asp:X}\quad\text{(Assumption~\tref{asp:X})}
        % &\leq C_x\norm{w-\Gamma(w)} + C_\tref{asp:W*}\\
        % &\leq C_\tref{lem wt bound}(\norm{w-\Gamma(w)}+1).
    \end{align}
    % where $C_\tref{lem wt bound}\doteq C_x\max\qty(1, C_\tref{asp:W*})$.
\end{proof}

\subsection{Proof of Lemma~\tref{lem bound H}}
\label{proof bound H}
\begin{proof}
    According to the definition of $H(w_t, Y_t)$ in \eqref{eq:H def},
    \begin{align*}
        &\norm{H(w_t, Y_t)}^2 \\
        \leq& C_\tref{asp:X}^2(\norm{Xw_t}+1)^2 \quad\text{(By Assumption~\tref{asp:X})}\\
        \leq& 2C_\tref{asp:X}^2(\norm{Xw_t}^2+1)\\
        \leq& 2C_\tref{asp:X}^2(C_\tref{lem wt bound}^2(\norm{w_t-\Gamma(w_t)}+1)^2+1)\\
        \leq& 2C_\tref{asp:X}^2(2C_\tref{lem wt bound}^2(\norm{w_t-\Gamma(w_t)}^2+1)+1)\\
        \leq& C_{\tref{lem bound H}}(\norm{w_t-\Gamma(w_t)}^2+1),
    \end{align*}
    where $C_{\tref{lem bound H}}\doteq 2C_\tref{asp:X}^2(2C_\tref{lem wt bound}^2+1)$. This completes the proof.
\end{proof}

\subsection{Proof of Lemma~\ref{lem:bound T}}
\label{proof:bound T}
\begin{proof}
    We first decompose $\langle w_t-\Gamma(w_t), H(w_t,Y_t)-h(w_t)\rangle$ into three components similarly to \citet{srikant2019finite} as
    % as in \cite{zhang2022globaloptimalityfinitesample}:
    \begin{align}
    \label{eq:T lambda}
        &\langle w_t-\Gamma(w_t), H(w_t,Y_t)-h(w_t)\rangle\\
        =& \underbrace{\langle  (w_{t} - \Gamma (w_{t})) -(w_{t-\tau_{\alpha_t}} - \Gamma(w_{t-\tau_{\alpha_t}})), H(w_t, Y_t)-h(w_{t})\rangle}_{T_1} \notag\\
        &+ \underbrace{\langle  w_{t-\tau_{\alpha_t} } - \Gamma (w_{t-\tau_{\alpha_t} } ), H(w_t, Y_t)-H(w_{t-\tau_{\alpha_t} }, Y_t) + h(w_{t-\tau_{\alpha_t} }) - h(w_t)\rangle}_{T_2} \notag\\
        &+\underbrace{\langle w_{t-\tau_{\alpha_t}} -\Gamma (w_{t-\tau_{\alpha_t}} ), H(w_{t-\tau_{\alpha_t} } , Y_t ) - h(w_{t-\tau_{\alpha_t} } ) \rangle}_{T_3}.
    \end{align}
    % Here, $T_1$, $T_2$ and $T_3$ denote recent changes, update function changes, and mixing error, respectively.
    We leverage Lemma~\tref{lem wt bound} and \eqref{eq:tau_alpha} to bound them one by one as follows.\\
    \tb{Bounding $T_1$:}
    \begin{align*}
        T_1\leq\underbrace{\norm{(w_t-\Gamma(w_t)) - (w_{t-\tau_{\alpha_t}} - \Gamma(w_{t-\tau_{\alpha_t}}))}}_{T_{11}} \cdot \underbrace{\norm{H(w_t,Y_t)-h(w_t)}}_{T_{12}}.
    \end{align*}
    For the first term, we have
    \begin{align}
    \label{eq:T11}
        T_{11}&= \norm{w_t-\Gamma(w_t)-w_{t-\tau_{\alpha_t}}-\Gamma(w_{t-\tau_{\alpha_t}})}\notag\\
        &\leq \norm{w_t-w_{t-\tau_{\alpha_t}}}+\norm{\Gamma(w_t)-\Gamma(w_{t-\tau_{\alpha_t}})}\notag\\
        &\leq 2\norm{w_t-w_{t-\tau_{\alpha_t}}}\notag\quad\text{(Since $W_*$ is convex)}\\
        &\leq 2C_\tref{bound_wt_wt1}\alpha_{t-\tau_{\alpha_t},t-1}(\norm{w_t-\Gamma(w_t)} +1)\quad\text{(Lemma~\tref{bound_wt_wt1})}
    \end{align}
    For the second term, we have
    \begin{align}
            T_{12} &\leq C_\tref{asp:X}(\norm{Xw_t}+1) + C_\tref{asp:X}(\norm{Xw_t}+1)\\
            &\leq 2C_\tref{asp:X}(C_\tref{lem wt bound}(\norm{w_t-\Gamma(w_t)}+1)+1)\notag\\
            % &\leq 2C_\tref{asp:X}C_{\tref{lem:bound T1},1} (\norm{w_t-\Gamma(w_t)} + 1)\\
            &\leq C_{\tref{lem:bound T},1}(\norm{w_t-\Gamma(w_t)}+1),
    \end{align}
    where $C_{\tref{lem:bound T},1}\doteq 2C_\tref{asp:X}(C_\tref{lem wt bound}+1)$.
    Therefore, we can get
    \begin{align*}
        T_1&\leq 2C_\tref{bound_wt_wt1}C_{\tref{lem:bound T},1}\alpha_{t-\tau_{\alpha_t},t-1}(\norm{w_t-\Gamma(w_t)} + 1)^2.
    \end{align*}
    Choosing $C_{\tref{lem:bound T},a} \doteq 4C_\tref{bound_wt_wt1}C_{\tref{lem:bound T},1}$ then yields the bound
    \begin{align*}
        T_1 \leq C_{\tref{lem:bound T},a}\alpha_{t-\tau_{\alpha_t},t-1}(\norm{w_t-\Gamma(w_t)}^2 + 1). 
    \end{align*}
    % \begin{lemma}[Bound on Recent Changes]
    % \label{lem:bound T1}
    % There exists a constant $C_\tref{lem:bound T1}$ such that
    % \begin{align*}
    %     T_1 \leq C_{\tref{lem:bound T1}}\alpha_{t-\tau_{\alpha_t},t-1}(\norm{w_t-\Gamma(w_t)} + 1)^2. 
    % \end{align*}
    % \end{lemma}
    % The proof is in Section \ref{proof:bound T1}.
    % \begin{lemma}[Bound on Update Function Changes]
    % \label{lem:bound T2}
    % There exists a constant $C_\tref{lem:bound T2}$ such that
    % \begin{align*}
    %     T_2 \leq C_\tref{lem:bound T2}\alpha_{t-\tau_{\alpha_t},t-1}(\norm{w_t -\Gamma (w_t)} + 1)^2.
    % \end{align*}
    % \end{lemma}
    % The proof is in Section \ref{proof:bound T2}.
\textbf{Bounding $T_2$:}
\begin{align*}
    T_2 &= \langle w_{t-\tau_{\alpha_t} } - \Gamma (w_{t-\tau_{\alpha_t} } ), H(w_t, Y_t)-H(w_{t-\tau_{\alpha_t} }, Y_t) + h(w_{t-\tau_{\alpha_t} }) - h(w_t)\rangle \\
    &\leq \underbrace{\norm{w_{t-\tau_{\alpha_t} } - \Gamma (w_{t-\tau_{\alpha_t} } )}}_{T_{21}} \cdot \underbrace{\norm{H(w_t, Y_t)-H(w_{t-\tau_{\alpha_t} }, Y_t) + h(w_{t-\tau_{\alpha_t} }) - h(w_t)\rangle}}_{T_{22}}.
\end{align*}

For the first term, we have:
\begin{align}
\label{eq:T21}
    \notag T_{21}
    &= \norm{(w_{t-\tau_{\alpha_t} } - \Gamma (w_{t-\tau_{\alpha_t} } )) - (\Gamma (w_t )- \Gamma (w_t ))} \notag\\
    \notag&\leq \norm{w_{t-\tau_{\alpha_t} } - \Gamma (w_t)} + \norm{\Gamma (w_t )- \Gamma (w_{t-\tau_{\alpha_t} } )} \notag\\
    &\leq \norm{w_{t-\tau_{\alpha_t} } - \Gamma (w_t)} + \norm{w_t - w_{t-\tau_{\alpha_t}}}\notag\\
    &\leq \norm{w_t - \Gamma (w_t)+w_{t-\tau_{\alpha_t} } - w_t} + \norm{w_t - w_{t-\tau_{\alpha_t}}}\notag \\
    &\leq \norm{w_t - \Gamma (w_t)} + 2\norm{w_t - w_{t-\tau_{\alpha_t}}}\notag  \\
    &\leq \norm{w_t - \Gamma (w_t)} + 2C_\tref{bound_wt_wt1}\alpha_{t-\tau_{\alpha_t},t-1}(\norm{w_t-\Gamma(w_t)} + 1) \quad\text{(Lemma~\tref{bound_wt_wt1})}\notag\\
    &\leq C_{\tref{lem:bound T},2}(\norm{w_t-\Gamma(w_t)} + 1).\quad\text{(Lemma~\tref{bound_tau_alpha})}
\end{align}

For the second term, we have:
\begin{align}
    T_{22} \leq &\norm{H(w_t, Y_t)-H(w_{t-\tau_{\alpha_t} }, Y_t)} + \norm{h(w_t)-h(w_{t-\tau_{\alpha_t} )}} \notag \\
    \leq& 2C_\tref{asp:lipH}\norm{w_{t-\tau_{\alpha_t} } - w_t}\notag \\
    \leq & C_{\tref{lem:bound T},3}C_\tref{bound_wt_wt1}\alpha_{t-\tau_{\alpha_t},t-1}(\norm{w_t -\Gamma (w_t )} + 1) .\quad\text{(Lemma~\tref{bound_wt_wt1})}\label{eq:T22}
\end{align}
Combine the result in \eqref{eq:T21} and \eqref{eq:T22}, we have:
\begin{align*}
    T_2 &\leq C_{\tref{lem:bound T},2}C_{\tref{lem:bound T},3}C_\tref{bound_wt_wt1}\alpha_{t-\tau_{\alpha_t},t-1}(\norm{w_t-\Gamma(w_t)} + 1)^2.
\end{align*}
Choosing $C_{\tref{lem:bound T},b} \doteq 2C_{\tref{lem:bound T},2}C_{\tref{lem:bound T},3}C_\tref{bound_wt_wt1}$  then yield the bound
    \begin{align*}
        T_2 \leq C_{\tref{lem:bound T},b}\alpha_{t-\tau_{\alpha_t},t-1}(\norm{w_t-\Gamma(w_t)}^2 + 1). 
    \end{align*}

    % \begin{lemma}[Bound on Mixing Error]
    % \label{lem:bound T3}
    % There exists a constant $C_\tref{lem:bound T3}$ such that
    % \begin{align*}
    %     \E\qty[T_{3}] \leq C_\tref{lem:bound T3}\alpha_t(\norm{w_t -\Gamma (w_t)} + 1)^2.
    % \end{align*}
    % \end{lemma}
    % The proof is in Section \ref{proof:bound T3}.

    % Denote $C_{\tref{lem:bound T}} \doteq C_\tref{lem:bound T1} + C_\tref{lem:bound T2} + C_\tref{lem:bound T3}$ then completes the proof.
\textbf{Bounding $T_3$:}
\begin{align*}
    T_3 &= \indot{w_{t-\tau_{\alpha_t} } -\Gamma (w_{t-\tau_{\alpha_t} } )}{H(w_{t-\tau_{\alpha_t} } , Y_t ) - h(w_{t-\tau_{\alpha_t} } )}.
\end{align*}
Take expectation on both sides, we can get
\begin{align*}
    \E \qty[T_3] &= \E\qty[\langle w_{t-\tau_{\alpha_t} } - \Gamma (w_{t-\tau_{\alpha_t} } ), H(w_{t-\tau_{\alpha_t} }, Y_t ) - h(w_{t-\tau_{\alpha_t} })\rangle ]\\
    &= \E\qty[\E\qty[\langle w_{t-\tau_{\alpha_t} } - \Gamma (w_{t-\tau_{\alpha_t} } ), H(w_{t-\tau_{\alpha_t} }, Y_t ) - h(w_{t-\tau_{\alpha_t} })\rangle ] \big|_{Y_{t-\tau_{\alpha_t}}}^{w_{t-\tau_{\alpha_t}}} ] \\
    &=\E\qty[\langle w_{t-\tau_{\alpha_t} } - \Gamma (w_{t-\tau_{\alpha_t} } ), \E\qty[H(w_{t-\tau_{\alpha_t} }, Y_t ) - h(w_{t-\tau_{\alpha_t} }) \big|_{Y_{t-\tau_{\alpha_t}}}^{w_{t-\tau_{\alpha_t}}}] \rangle ] \\
    &\leq \E\qty[\underbrace{\norm{w_{t-\tau_{\alpha_t} } - \Gamma (w_{t-\tau_{\alpha_t} })}}_{T_{31}} \cdot \underbrace{\norm{\E\qty[ H(w_{t-\tau_{\alpha_t} }, Y_t ) - h(w_{t-\tau_{\alpha_t} })|_{Y_{t-\tau_{\alpha_t}}}^{w_{t-\tau_{\alpha_t}}} ]}}_{T_{32}} ].
\end{align*}
We have 
\begin{align}
    T_{32} \leq& \alpha_t (\norm{Xw_{t-\tau_{\alpha_t} }}+1)\qq{(By \eqref{eq:uniform_mixing_in_A1} and \eqref{eq:tau_alpha})}\\
    \leq &\alpha_t (\norm{Xw_{t-\tau_{\alpha_t} } - Xw_t} +\norm{Xw_t} +1) \\
    \leq &\alpha_t (\norm{Xw_{t-\tau_{\alpha_t} } - Xw_t} + C_\tref{lem wt bound}(\norm{w_t-\Gamma(w_t)}+1)+1) \\
    \leq &\alpha_t (C_\tref{bound_wt_wt1}C_x\alpha_{t-\tau_{\alpha_t},t-1}(\norm{w_t-\Gamma(w_t)} +1)+C_\tref{lem wt bound}\norm{w_t-\Gamma(w_t)}+C_\tref{lem wt bound} +1)\\
    \leq& C_{\tref{lem:bound T},4}\alpha_t(\norm{w_t -\Gamma (w_t)}+1).  
\end{align}
% \begin{align*}
%     % &\norm{\E\qty[H(w_{t-\tau_{\alpha_t} }, Y_t ) - h(w_{t-\tau_{\alpha_t} })|_{Y_{t-\tau_{\alpha_t}}}^{w_{t-\tau_{\alpha_t}}} ]} \\
%      T_{32} = & \norm{\sum_y \qty(P(Y_t = y|_{Y_{t-\tau_t}}^{w_{t-\tau_t}} ) - d_\pi(y))H(w_{t-\tau_{\alpha_t} } , y)} \\
%      \leq &\sup\limits_y \norm{H(w_{t-\tau_{\alpha_t} } , y)} \times \norm{\sum_y \qty(P(Y_t = y|_{Y_{t-\tau_t}}^{w_{t-\tau_t}} ) - d_\pi(y))} \\
%     \leq &C_\tref{asp:X}(\norm{Xw_{t-\tau_{\alpha_t} }}+1) \alpha_t\quad\text{(By Assumption~\tref{asp:X}, \eqref{eq:uniform_mixing_in_A1}, and \eqref{eq:tau_alpha})} \\
%     \leq &C_\tref{asp:X}\alpha_t (\norm{Xw_{t-\tau_{\alpha_t} } - Xw_t} +\norm{Xw_t} +1) \\
%     \leq &C_\tref{asp:X}\alpha_t (\norm{Xw_{t-\tau_{\alpha_t} } - Xw_t} + C_\tref{lem wt bound}(\norm{w_t-\Gamma(w_t)}+1)+1) \\
%     \leq &C_\tref{asp:X}\alpha_t (C_\tref{bound_wt_wt1}C_x\alpha_{t-\tau_{\alpha_t},t-1}(\norm{w_t-\Gamma(w_t)} +1)+C_\tref{lem wt bound}\norm{w_t-\Gamma(w_t)}+C_\tref{lem wt bound} +1)\\
%     \leq& C_{\tref{lem:bound T},4}\alpha_t(\norm{w_t -\Gamma (w_t)}+1).    
% \end{align*}     
Thus, together with \eqref{eq:T21}, 
we obtain
% consider the boundary of $\E[T_3]$, 
% we obtain the following inequality
\begin{align*}
    \E \qty[T_3] &\leq C_{\tref{lem:bound T},c}\alpha_t(\norm{w_t-\Gamma(w_t)}^2 + 1),
\end{align*}
  where $C_{\tref{lem:bound T},c} \doteq C_{\tref{lem:bound T},2}C_{\tref{lem:bound T},4}$.
  Finally, denote $C_{\tref{lem:bound T}} \doteq C_{\tref{lem:bound T},a} + C_{\tref{lem:bound T},b} + C_{\tref{lem:bound T},c}$ then completes the proof.
\end{proof}

\subsection{Proof of Lemma~\tref{lem sa recur}}
\label{proof sa recur}
\begin{proof}
    % Since $T_1$ and $T_2$ have deterministic upper bounds, we can directly substitute their upper bounds into the expectation. To further simplify, we can express $\norm{w_t - \Gamma(w_t)}$ in terms of $L(w_t)$. Given Assumption~\tref{asp:a3}, we have:
    We recall that
    \begin{align*}
        \norm{w_t - \Gamma(w_t)}^2 = 2L(w_t).
    \end{align*}
     Aligning Assumption~\tref{asp:negdrift}, Lemmas~\ref{lem bound H} and \ref{lem:bound T} with \eqref{eq L expand},
    we get 
    % \sz{You should algin the RHS of Lemmas 2 and 3 to get rid of $\sqrt{}$ directly.}
    \begin{align*}
    &\E\qty[L(w_{t+1})] \\
    \leq &(1-C_\tref{asp:negdrift} \alpha_t)\E\qty[L(w_t)] + C_\tref{lem:bound T} \alpha_t \alpha_{t-\tau_{\alpha_t},t-1} \qty( 2\E\qty[L(w_t)] + 1) + \frac{C_{\tref{lem bound H}}}{2} \alpha_t^2  + C_{\tref{lem bound H}} \alpha_t^2 \E\qty[L(w_t)]\\
    \leq& \qty(1-2C_\tref{asp:negdrift} \alpha_t +2C_\tref{lem:bound T}\alpha_t\alpha_{t-\tau_{\alpha_t} ,t-1} + C_{\tref{lem bound H}}\alpha_t^2)\E\qty[L(w_t)]+ C_\tref{lem:bound T}\alpha_t\alpha_{t-\tau_{\alpha_t},t-1} +  \frac{C_{\tref{lem bound H}}}{2}\alpha_t^2.
    \end{align*}
    % where the last inequality is obtained by Cauchy-Schwartz inequality.
    
    Furthermore, we aim to derive an upper bound for $\E\qty[L(w_t)]$ that depends on the initial expected loss $\E\qty[L(w_0)]$ and decreases over time. First, let's denote the coefficients as $C_t$ and $D_t$:
    \begin{align*}
        &C_t \doteq 1-2C_\tref{asp:negdrift} \alpha_t +2C_\tref{lem:bound T}\alpha_t\alpha_{t-\tau_{\alpha_t} ,t-1} + C_{\tref{lem bound H}}\alpha_t^2,\\
        &D_t \doteq C_\tref{lem:bound T}\alpha_t\alpha_{t-\tau_{\alpha_t},t-1} + \frac{C_{\tref{lem bound H}}}{2}\alpha_t^2.
    \end{align*}
    For sufficiently large $t_0$ and $t \geq \overline t$ 
    % \sz{and $t \geq \overline t$, check this everywhere}
    , we obtain $4C_\tref{lem:bound T}\alpha_{t-\tau_{\alpha_t} ,t-1} + C_{\tref{lem bound H}}\alpha_t < C_\tref{asp:negdrift}$.
    Thus, the recursive inequality further becomes:
    \begin{align*}
        \E\qty[L(w_{t+1})] \leq (1 - C_\tref{asp:negdrift} \alpha_t) \E\qty[L(w_t)] + D_t,
    \end{align*}
    where $D_t=\fO(\alpha_t\alpha_{t-\tau_{\alpha_t},t-1})$.
\end{proof}

\subsection{Proof of Theorem~\tref{thm:sa}}
\label{proof:sa}
\begin{proof}
    % \sz{I think again you need to start from $\overline t$ instead of $t=0$, as in the ICML paper.}
To express $\E\qty[L(w_t)]$ in terms of $\E\qty[L(w_0)]$, we recursively apply the inequality:
\begin{align*}
    \E\qty[L(w_t)] \leq \prod_{i=\overline t}^{t} (1 - C_\tref{asp:negdrift} \alpha_i) \E\qty[L(w_{\overline t})] + \sum_{j=\overline t}^{t} \qty( \prod_{i=j+1}^{t} (1 - C_\tref{asp:negdrift} \alpha_i)) D_j.
\end{align*}
 Denote $E_1 \doteq \prod_{i=\overline t}^{t} (1 - C_\tref{asp:negdrift} \alpha_i) \E[L(w_{\overline t})]$, $E_2 \doteq \sum_{j=\overline t}^{t} \left( \prod_{i=j+1}^{t} (1 - C_\tref{asp:negdrift} \alpha_i) \right) \frac{\ln(j+t_0)}{(j+t_0)^{2\xi}}$, and $\kappa = C_\tref{asp:negdrift} \alpha$. Recall we have $\alpha_t=\frac{\alpha}{(t+t_0)^\xi}$.
For $E_1$, set $t_0 >\kappa= C_\tref{asp:negdrift} \alpha$, we have
\begin{align*}
    \prod_{i=\overline t}^{t} (1 - C_\tref{asp:negdrift} \alpha_i) \E\qty[L(w_{\overline t})] =&\prod_{i=\overline t}^{t} \qty(1 - \frac{C_\tref{asp:negdrift} \alpha}{(i+t_0)^\xi}) \E\qty[L(w_{\overline t})] \\
    \leq& \prod_{i=\overline t}^{t} \qty(1 - \frac{\kappa}{i+t_0}) \E\qty[L(w_{\overline t})] \\
    =& \E\qty[L(w_{\overline t})] \prod_{i=\overline t}^{t} \frac{i+t_0-\kappa}{i+t_0} \\
    \leq& \E\qty[L(w_{\overline t})] \qty(\frac{\overline t+ t_0}{t+t_0-\kappa})^{\lfloor \kappa \rfloor}. 
\end{align*}

For $E_2$, we have
\begin{align*}
    E_2=& \sum_{j={\overline t}}^{t} \qty( \prod_{i=j+1}^{t} \frac{i+t_0-\kappa}{i+t_0}) \frac{\ln(j+t_0)}{(j+t_0)^{2\xi}}\\
    =& \sum_{j={\overline t}}^{t-\lfloor \kappa \rfloor} \qty( \prod_{i=j+1}^{t} \frac{i+t_0-\kappa}{i+t_0}) \frac{\ln(j+t_0)}{(j+t_0)^{2\xi}}+ \sum_{j=t-\lfloor \kappa \rfloor+1}^{t} \qty( \prod_{i=j+1}^{t} \frac{i+t_0-\kappa}{i+t_0}) \frac{\ln(j+t_0)}{(j+t_0)^{2\xi}}\\
    \leq& \sum_{j={\overline t}}^{t-\lfloor \kappa \rfloor} \qty( \frac{j+1+t_0}{t+t_0-\kappa})^{\lfloor \kappa \rfloor} \frac{\ln(j+t_0)}{(j+t_0)^{2\xi}} + \lfloor \kappa \rfloor \frac{\ln(t+t_0)}{(t-\lfloor \kappa \rfloor+1+t_0)^{2\xi}}\\
    \leq& \frac{\ln(t+t_0)}{(t+t_0-\kappa)^{\lfloor \kappa \rfloor}}C_\text{Thm\ref{thm:sa},1}\sum_{j={\overline t}}^{t-\lfloor \kappa \rfloor} (j+t_0)^{\lfloor \kappa \rfloor-2\xi} + \lfloor \kappa \rfloor \frac{\ln(t+t_0)}{(t-\kappa+1+t_0)^{2\xi}}
\end{align*}
\textbf{Case 1: $\lfloor \kappa \rfloor-2\xi>0$}
\begin{align*}
    E_2 \leq& \frac{\ln(t+t_0)}{(t+t_0-\kappa)^{\lfloor \kappa \rfloor}}C_\text{Thm\ref{thm:sa},2}(t-\lfloor \kappa \rfloor+t_0)^{\lfloor \kappa \rfloor-2\xi +1} + \lfloor \kappa \rfloor \frac{\ln(t+t_0)}{(t-\kappa+1+t_0)^{2\xi}}\\
    \leq& \frac{\ln(t+t_0)}{(t+t_0-\kappa)^{2\xi-1}}C_\text{Thm\ref{thm:sa},3} + \lfloor \kappa \rfloor \frac{\ln(t+t_0)}{(t-\kappa+1+t_0)^{2\xi}}\\
    \leq& C_\text{Thm\ref{thm:sa},4}\qty(\frac{\ln(t+t_0)}{(t+t_0)^{2\xi-1}}).
\end{align*}
\textbf{Case 2: $\lfloor \kappa \rfloor-2\xi\leq0$}
\begin{align*}
    E_2 \leq& \frac{\ln(t+t_0)}{(t+t_0-\kappa)^{\lfloor \kappa \rfloor}}C_\text{Thm\ref{thm:sa},1}(t-\lfloor \kappa \rfloor+1) + \lfloor \kappa \rfloor \frac{\ln(t+t_0)}{(t-\kappa+1+t_0)^{2\xi}}\\
    \leq& \frac{\ln(t+t_0)}{(t+t_0-\kappa)^{\lfloor \kappa \rfloor-1}}C_\text{Thm\ref{thm:sa},5} + \lfloor \kappa \rfloor \frac{\ln(t+t_0)}{(t-\kappa+1+t_0)^{2\xi}}\\
    \leq& C_\text{Thm\ref{thm:sa},6}\qty(\frac{\ln(t+t_0)}{(t+t_0)^{\lfloor \kappa \rfloor-1}}).
\end{align*}

Starting from the update of $w_{t+1}$, we have 
\begin{align}
    \norm{w_{t+1}}\leq \norm{w_t} + \alpha_t \norm{H(w_t,Y_{t+1})}
    \leq \norm{w_t} + \alpha_t C_\tref{asp:lipH}(\norm{w_t}+1).
\end{align}
That is, $\norm{w_{t+1}} \leq \alpha_0 C_\tref{asp:lipH} + \sum_{i=0}^t(\alpha_0 C_\tref{asp:lipH}+1)\norm{w_i}$. Applying discrete Gronwall inequality, we obtain
$\norm{w_{\overline{t}}} \leq (C_\tref{asp:lipH} + \norm{w_0}) \exp(\sum_{t=0}^{\overline{t}-1} (1+\alpha_0 C_\tref{asp:lipH})) = (C_\tref{asp:lipH} + \norm{w_0}) \exp(\overline{t}+\overline{t}\alpha_0 C_\tref{asp:lipH})$. 

Denoting $C_\text{Thm\ref{thm:sa},1}\doteq\exp(2\overline{t}+2\overline{t}\alpha_0 C_\tref{asp:lipH})$ and $C_\text{Thm\ref{thm:sa},2}\doteq 2\max(C_\text{Thm\ref{thm:sa},4}, C_\text{Thm\ref{thm:sa},6})$
then completes the proof.
\end{proof}

\section{Proofs in Section~\tref{sec TD rate}}
\label{sec proof td lambda}

\subsection{Proof of Lemma~\tref{assu Lipschitz}}
\label{proof assu Lip}
\begin{proof}
% Recall that $H(w, y) = (R_{t+1} + \gamma x_{t+1}^T w - x_t^T w) e_t$, $C_x \doteq \max_{s} \norm{x(s)}$. 
% The eligibility trace $e_t$ is geometrically bounded
% \begin{align}
% \label{eq:et bound}
    % \norm{e_t} &= \norm{x_t + \lambda\gamma x_{t-1} + ... + (\lambda\gamma)^t x_0} \leq \sum_{k=0}^t (\lambda\gamma)^k \norm{x_{t-k}} \leq \frac{C_x}{1-\lambda\gamma}. 
% \end{align}
% Let $C_e \doteq \frac{C_x}{1-\lambda\gamma}$, we have
% we have
Let $y = (s, a, s', e) \in \fY$ and $C_x \doteq \max_s \norm{x(s)}$.
We have
\begin{align*}
    \norm{H(w,y) - H(w',y)} &= \norm{e(\gamma x(s')^\top - x(s)^\top)(w - w')} \leq 2C_xC_e\norm{w-w'}.
\end{align*}
Furthermore,
% define $C_R \doteq \max_{s, a}\abs{r(s, a)}$.
% Then we have  
\begin{align}
  \sup_{y \in \mathcal{Y}} \norm{H(0, y)} = \sup_{y \in \mathcal{Y}} \norm{ r(s, a)e } \leq \max_{s, a}\abs{r(s, a)} C_e,
\end{align}
which completes the proof.
% $$. Choosing $C_\tref{assu Lipschitz} \doteq \max(C_xC_e, C_R C_e)$ then completes the proof.
\end{proof}

\subsection{Proof of Lemma~\tref{lem:td_mix}}
\begin{lemma}
    \label{lem:td_mix}
    There exist a constant $C_{\tref{lem:td_mix}}$ and $\tau \in [0,1)$ such that $\forall w$
    \begin{align}
        \norm{\E[H(w,Y_{t+n})\, | \, Y_t]-h(w)}\leq C_{\tref{lem:td_mix}}\tau^n (\norm{Xw}+1).
    \end{align}
\end{lemma}
\begin{proof}
    Given the Markov property, we only need to prove the case of $t=1$.
    Recall that we use $y = (s, a, s', e)$.
    Define shorthand
    \begin{align}
        \delta((s, a, s'), w) \doteq& r(s, a) + \gamma x(s')^\top w - x(s)^\top w, \\
        \delta_{n+1}(w) \doteq& \delta((S_n, A_n, S_{n+1}), w).
    \end{align} 
By \eqref{eq:H def}, we can get
\begin{align}
    H(w, Y_{n+1}) =& \delta_{n+1}(w)e_n.
\end{align}
By expanding $e_n$, 
we get
% we can write $\E[H(w,Y_{n+1})\, | \, Y_1]$ as 
% \sz{Shouldn't the below be $S_{t+k}$ instead of $S_k$?}
\begin{align}
    &\E[H(w,Y_{n+1})\, | \, Y_1] \\
    =&\E\qty[\delta_{n+1}(w) e_{n} \, | \, Y_1] \\
    =& \E\qty[\delta_{n+1}(w)\sum_{k=0}^{n}(\gamma \lambda)^{n-k} x(S_k)\, | \, S_0].
    % =& \E_{d_\mathcal{Y}}\qty[\sum_{k=-\infty}^{t+n}(\gamma \lambda)^{t+n-k}\delta_{t+n}(w) x(S_k)]+f_1(n) + f_2(n)\\
    % =& \E_{d_\mathcal{Y}}\qty[\delta_{t+n}(w) e_{t+n}] +f_1(n) + f_2(n),
\end{align}
Now define a two-sided Markov chain $\qty{\bar S_t, \bar A_t}_{t=\dots, -2, -1, 0, 1, 2, \dots}$ such that $\Pr(\bar S_t = s) = d_\pi(s), \Pr(\bar A_t = a | \bar S_t = s) = \pi(a|s)$, i.e., the new chain always stay in the stationary distribution of the original chain.
Similarly,
define
\begin{align}
    \bar \delta_{n+1}(w) \doteq \delta((\bar S_n, \bar A_n, \bar S_{n+1}), w).
\end{align}
We then have 
\begin{align}
    &\E\qty[\delta_{n+1}(w)\sum_{k=0}^{n}(\gamma \lambda)^{n-k} x(S_k)\, | \, S_0] \\
    =& \underbrace{\E\qty[\bar \delta_{n+1}(w)\sum_{k=-\infty}^{n}(\gamma \lambda)^{n-k} x(\bar S_k)]}_{f_0(n)} \\
    &+ \underbrace{\E\qty[\delta_{n+1}(w)\sum_{k=0}^{n}(\gamma \lambda)^{n-k} x(S_k)\, | \, S_0] - \E\qty[\bar \delta_{n+1}(w)\sum_{k=0}^{n}(\gamma \lambda)^{n-k} x(\bar S_k)]}_{f_1(n)} \\
    &- \underbrace{\E\qty[\bar \delta_{n+1}(w)\sum_{k=-\infty}^{-1}(\gamma \lambda)^{n-k} x(\bar S_k)]}_{f_2(n)}.
\end{align}
In the proof of Lemma 6.7 of \citet{bertsekas1996neuro},
it is proved that
\begin{align}
    f_0(n) = Aw + b,
\end{align}
which coincides with $h(w)$.
Thus the rest of the proof is dedicated to proving that $f_1(n)$ and $f_2(n)$ decay geometrically.
% where $f_1(n)=\sum_{k=0}^{t+n} (\gamma \lambda)^{t+n-k}  \E[\delta_{t+n}(w) x(S_k) \, | \,  Y_t] - \sum_{k=0}^{t+n} (\gamma \lambda)^{n-k}\E_{d_\mathcal{Y}}[\delta_{t+n}(w) x(S_k)] $, and $f_2(n)=-\sum_{k=-\infty}^{-1} (\gamma \lambda)^{t+n-k} \mathbb{E}_{d_\mathcal{Y}}[\delta_{t+n}(w) x(S_k)]$.
% \begin{align}
%     \E[\delta_n(w) e_n | Y_1] - h(w) =& \sum_{k=0}^n (\gamma \lambda)^{n-k} \left( \E[\delta_n(w) x(S_k) | Y_1] 
%     - \E_{d_\mathcal{Y}}[\delta_n(w) x(S_k)] \right) \\
%     &+ (\gamma \lambda)^n \E[\delta_n(w) e_0 | Y_1] - \sum_{k=-\infty}^{-1} (\gamma \lambda)^{n-k} \E_{d_\mathcal{Y}}[\delta_n(w) x(S_k)].
% \end{align}
% where $f_1(n) = \sum_{k=0}^n (\gamma \lambda)^{n-k} \left( \mathbb{E}[\delta_n(w) x(S_k) \, | \,  Y_1] - \mathbb{E}_{d_\mathcal{Y}}[\delta_n(w) x(S_k)] \right)$, $f_2(n) = (\gamma \lambda)^n \mathbb{E}[\delta_n(w) e_0 \, | \,  Y_1]$, and $f_3(n) = -\sum_{k=-\infty}^{-1} (\gamma \lambda)^{n-k} \mathbb{E}_{d_\mathcal{Y}}[\delta_n(w) x(S_k)]$.  
% We will now show that $f_1(n)$ and $f_2(n)$ decay exponentially as $t$ tends to infinity. 
For $f_2(n)$, 
we have $\norm{\bar \delta_{n+1}(w) x(\bar S_k)} \leq C_{\tref{lem:td_mix},1}(\norm{Xw}+1)$ for some $C_{\tref{lem:td_mix},1}$ (cf.~\eqref{eq bound h with xw}).
We then have
\begin{align}
    \norm{f_2(n)} \leq& C_{\tref{lem:td_mix},1}(\norm{Xw}+1)\sum_{k=-\infty}^{-1} (\gamma\lambda)^{n-k} \\
    =& C_{\tref{lem:td_mix},1}(\norm{Xw}+1) (\gamma\lambda)^n \sum_{k=1}^\infty (\gamma\lambda)^k.
\end{align}
For $f_1(n)$,
since $\qty{S_t}$ adopts geometric mixing,
there exists some $\tau_1 \in [0, 1)$ and $C_{\tref{lem:td_mix},2} > 0$ such that
\begin{align}
    \sum_{s}\abs{\Pr(S_k = s) - \Pr(\bar S_k = s)} \leq C_{\tref{lem:td_mix},2}\tau_1^k.
\end{align}
Then we have
\begin{align}
    &\E\qty[\delta_{n+1}(w) x(S_k) | S_0] - \E\qty[\bar \delta_{n+1}(w) x(\bar S_k)] \\
    =& \sum_s \Pr(S_k = s | S_0) x(S_k) \E\qty[\delta_{n+1}(w) | S_k = s] - \sum_s d_\pi(s) x(\bar S_k) \E\qty[\bar \delta_{n+1}(w) | \bar S_k = s].
\end{align}
Noticing that $\E\qty[\delta_{n+1}(w) | S_k = s] = \E\qty[\bar \delta_{n+1}(w) | \bar S_k = s]$ due to the Markov property,
we obtain
\begin{align}
    \norm{\E\qty[\delta_{n+1}(w) x(S_k) | S_0] - \E\qty[\bar \delta_{n+1}(w) x(\bar S_k)]} \leq C_{\tref{lem:td_mix},2}\tau_1^k C_{\tref{lem:td_mix},1}(\norm{Xw}+1).
\end{align}
This means 
\begin{align}
    \norm{f_2(n)} \leq C_{\tref{lem:td_mix},2} C_{\tref{lem:td_mix},1}(\norm{Xw}+1) \sum_{k=0}^n (\gamma \lambda)^{n-k} \tau_1^k.
\end{align}
Noticing that 
\begin{align}
    \sum_{k=0}^n (\gamma \lambda)^{n-k} \tau_1^k \leq n \max\qty{\gamma\lambda, \tau_1}^n
\end{align}
then completes the proof.
\end{proof}

\subsection{Proof of Lemma~\ref{prop:a}}
\label{proof:a}
\begin{proof}
We start with proving $\forall w \in \ker(A)^\perp, w^\top A w \leq -C_\tref{prop:a}\norm{w}^2$.
This is apparently true if $w = \0$.
Now fix any $w \in \ker(A)^\perp$ and $w \neq \0$,
which implies that $Aw \neq \0$.
Now we prove by contradiction that $w^\top A w \neq 0$.
Otherwise,
if $w^\top A w = 0$,
we have $w^\top X^\top D_\pi(\gamma P_\lambda - I) Xw = 0$.
Since $D_\pi (\gamma P_\lambda - I)$ is n.d.,
we then get $Xw = \0$,
further implying $Aw = \0$,
which is a contradiction.
We have now proved that $w^\top Aw \neq 0$.
We next prove that $w^\top A w < 0$.
This is from the fact that $A$ is n.d., i.e., for $\forall z \in \R[d], z^\top A z \leq 0$.
But $w^\top A w \neq 0$.
So we must have $w^\top A w < 0$.
Finally, we use an extreme theorem argument to complete the proof.
Define $Z \doteq \qty{w | w \in \ker(A)^\perp, \norm{w} = 1}$.
Because $z \in Z$ implies $z \in \ker(A)^\perp$ and $z \neq 0$,
we have $\forall z \in Z, z^\top A z < 0$.
Since $Z$ is clearly compact,
the extreme value theorem confirms that the function $z \mapsto z^\top A z$ obtains its minimum value in $Z$,
denoted as $-C_\tref{prop:a} < 0$, i.e.,
we have 
\begin{align}
    \label{eq evt trick}
    \forall z \in Z, z^\top A z \leq -C_\tref{prop:a}.
\end{align}
For any $w \in \ker(A)^\perp$ and $w \neq \0$,
we have $\frac{w}{\norm{w}} \in Z$,
so $w^\top A w \leq -C_\tref{prop:a} \norm{w}^2$, which completes the proof of the first part.

We now prove that $\forall w \in \R[d], w - \Gamma(w) \in \ker(A)^\perp$.
We recall that $\Gamma$ is the orthogonal projection to $W_* = \qty{w \mid Aw + b = 0}$.
% where the second equality is from Theorem~1 \sz{x} of \citet{wang2024sureconvergencelineartemporal}.
Since $\Gamma$ is the orthogonal projection to $W_*$,
we know $w - \Gamma(w) \in W_*^{\perp}$.
Fix any $w_* \in W_*$ and
% we know $\forall w \in \R[d]$,
% it holds that $\inner{w - \Gamma(w)}{w_*_0} = 0$.
let $z \in \ker(A)$,
we then have $A(w_* + z) + b = \0$ so $w_* + z \in W_*$.
We then have
\begin{align}
    \inner{w - \Gamma(w)}{z} = \inner{w - \Gamma(w)}{w_* + z} - \inner{w - \Gamma(w)}{w_*} = 0 - 0 = 0,
\end{align}
confirming that $w - \Gamma(w) \in \ker(A)^\perp$,
which completes the proof.

\end{proof}
\subsection{Proof of Lemma~\tref{lem H h linear growth}}
\label{proof H h linear growth}
\begin{proof}
    % \sz{Remove $t$. Follow the previous lemma.}\zx{Done.}
    Let $y = (s, a, s', e) \in \fY$,
    since $\abs{x(s)^\top w} \leq \max_{s \in S} \abs{x(s)^\top w} \leq \norm{Xw}$,
    according to \eqref{eq:H def}, we have
    \begin{align}
        \label{eq bound h with xw}
        \norm{H(w, y)} &= \norm{e\qty(r(s,a) + \gamma x(s')^\top w - x(s)^\top w)}\\
        % &\leq \norm{e_t}\abs{r(s,a) + \gamma x(s')^T w - x(s)^T w}\\
        &\leq C_e(\abs{r(s,a)} + \gamma\abs{x(s')^\top w} + \abs{x(s)^\top w})\\
        &\leq C_e(C_R + (\gamma+1)\norm{X w})\\
        &\leq C_{\tref{lem H h linear growth}}(\norm{X w}+1),
    \end{align}
    where $C_{\tref{lem H h linear growth}}\doteq C_e(C_R + \gamma+1)$.
    For $\norm{h(w)}$, we have
    \begin{align}
        \norm{h(w)} = \norm{\E_{y\sim d_\mathcal{Y}}\qty[H(w, y)]} \leq  E_{y\sim d_\mathcal{Y}}[\norm{H(w, y)}] \leq C_{\tref{lem H h linear growth}}(\norm{X w}+1),
    \end{align}
    which completes the proof.
\end{proof}

\section{Proofs in Section~\ref{sec:ar}}
\label{sec proof tdar}

\subsection{Proof of Lemma~\tref{lem:9}}
\label{proof:9}
\begin{proof}
The update to $\qty{\hat J_t}$ in~\eqref{eq:artd} is 
% for \(\hat{J}_t\) is:
% \begin{align}
    % \hat{J}_{t+1} = \hat{J}_t + \beta_t \left( R_{t+1} - \hat{J}_t \right)= \hat{J}_t + c_\beta \alpha_t \left( R_{t+1} - \hat{J}_t \right).
% \end{align}
% Rewrite this in the form of the lemma:
\begin{align}
    \hat{J}_{t+1} = \hat{J}_t + \alpha_t \left( c_\beta R_{t+1} - c_\beta \hat{J}_t \right).
\end{align}
This matches the first row of 
\begin{align}
    \widetilde{A}(Y_t) \widetilde{w}_t + \widetilde{b}(Y_t) = \begin{bmatrix} -c_\beta & 0 \\ -\Pi  e_t & \Pi  e_t (x(S_{t+1})^\top - x(S_t)^\top) \end{bmatrix} \begin{bmatrix} \hat{J}_t \\ \Pi  w_t \end{bmatrix} + \begin{bmatrix} c_\beta R_{t+1} \\ R_{t+1} \Pi  e_t \end{bmatrix}.
\end{align}
Now consider the update for $w_t$
\begin{align}
    w_{t+1} = w_t + \alpha_t \left( R_{t+1} - \hat{J}_t + x(S_{t+1})^\top w_t - x(S_t)^\top w_t \right) e_t.
\end{align}
Applying the projection matrix $\Pi$ on both sides yields
\begin{align}
    \Pi w_{t+1} - \Pi  w_t=& \alpha_t \Pi  \left( \left( R_{t+1} - \hat{J}_t + x(S_{t+1})^\top w_t - x(S_t)^\top w_t \right) e_t \right)\\
    =& \left( R_{t+1} - \hat{J}_t + x(S_{t+1})^\top w_t - x(S_t)^\top w_t \right) \Pi  e_t\\
    =& \left( R_{t+1} - \hat{J}_t + x(S_{t+1})^\top \Pi w_t - x(S_t)^\top \Pi w_t \right) \Pi  e_t.
\end{align}
To see the last equality,
we recall Lemma~\ref{lem:z1z2} and
recall $\Pi = X_1^\dagger X_1$.
We then have
\begin{align}
    X \Pi w =& X_1 \Pi w + \1\theta^\top \Pi w \\
    =& X_1 w + \1\theta^\top \Pi w.
\end{align}
This means that
\begin{align}
    x(s')^\top \Pi w - x(s)^\top \Pi w = x_1(s')^\top w - x_1(s)^\top w,
\end{align}
where we use $x_1(s)$ to denote the $s$-th row of $X_1$.
We also have
\begin{align}
    x(s')^\top w - x(s)^\top w =& (x_1(s') + \theta)^\top w - (x(s) + \theta)^\top w \\
    =&x_1(s')^\top w - x_1(s)^\top w,
\end{align}
which confirms the last equality and then completes the proof.
% let $x_1(s)$ be the $s$-th row of $X_1$.
% We then have $x(s) = x_1(s) + \theta$.
% It then holds that
% \begin{align}
%     &x(s')^\top \Pi w - x(s)^\top \Pi w \\
%     =& (x_1(s') - x_1(s))^\top \Pi w
% \end{align}
% The last step holds if we decompose $x(s)=x_1(s) + x_2(s)$, where $x_1(s)\perp \1$ and $x_2(s) \parallel\1$, then $x_1(S_t)^\top w_t=x_1(S_{t+1})^\top w_t=0$, $x_2(S_t)-x_2(S_{t+1})=0$. 
% This matches the second row of the target update, hence completes the proof.
\end{proof}

\subsection{Proof of Lemma~\tref{lem:gamma ar}}
\begin{lemma}
\label{lem:gamma ar}
    $\widetilde A \Gamma(\tw) + \widetilde b = \0$
\end{lemma}
\begin{proof}
    According to the definition of $\Gamma(\tw)$, $\Gamma(\tw) \in \widetilde W_* \doteq \left\{\begin{bmatrix} J_\pi \\ \Pi  w \end{bmatrix}\middle| w \in \overline W_*\right\}$.
    We have
    \begin{align}
    \label{eq:expA}
        \notag &\widetilde{A} = \E_{y\sim d_\mathcal{Y}} \qty[\pA(y)]
        = \E_{(s, a, s', e)\sim d_\mathcal{Y}} \begin{bmatrix}
        -c_\beta & \0\\
        -\Pi e & \Pi \qty(e(x(s')^\top - x(s)^\top)) \end{bmatrix}
        =\begin{bmatrix} -c_\beta & \0 \\ -\Pi \E_{d_\fY}[e] & \Pi \overline A \end{bmatrix}, \\
        % &=\begin{bmatrix} -c_\beta & 0 \\ -\Pi \E_{d_\fY}[e] & \Pi X_1^\top D_\pi(P_\lambda - I)X_1 \end{bmatrix} \qq{(Lemma~\ref{lem:A bar})},\\
        &\widetilde b = \E_{y \sim d_\fY} \qty[\pB(y)]
        = \E_{(s, a, s', e)\sim d_\mathcal{Y}} \begin{bmatrix}
            c_\beta r(s, a) \\
            r(s, a) \Pi e \end{bmatrix}
        =\begin{bmatrix}
            c_\beta J_\pi \\
            \Pi \E_{d_\fY}[e]J_\pi + \Pi \overline b \end{bmatrix}.
    \end{align}
    Therefore, for the first row of $\widetilde A \Gamma(\tw) + \widetilde b$ , we get $c_\beta (J_\pi-J_\pi) = 0$. 
    For the second row, we can get 
    \begin{align}
        &-\Pi \E_{d_\fY}[e]J_\pi + \Pi \overline A \Pi w + \Pi \E_{d_\fY}[e]J_\pi + \Pi \overline b \\
        =&\Pi(\overline A \Pi w+\overline b) \\
        =& \Pi (X_1^\top D_\pi(P_\lambda-I)X_1 \Pi w + \overline b)\\
        =& \Pi (X_1^\top D_\pi(P_\lambda-I)X_1 w + \overline b)\\
        =& \Pi (\overline A w+\overline b)\\
        =&\0,
    \end{align}
    where the second equality comes with the definition of $\Pi$. 
    This completes the proof.
\end{proof}

\subsection{Proof of Lemma~\ref{lem negdef pA}}
\label{proof:negdef pA}
\begin{proof}
    If $z = 0$, the lemma trivially holds. So now let 
Let $z = \begin{bmatrix} z_1 \\ z_2 \end{bmatrix} \in \mathbb{R} \times \ker(X_1)^\perp$, $z \neq 0$. 
% According to the definition of $\Pi $, 
With \eqref{eq:expA}, we have
\begin{align}
    \widetilde{A} 
    % &= \E_{y\sim d_\mathcal{Y}} \qty[\pA(y)]\\
    % &= \E_{(s, a, s', e)\sim d_\mathcal{Y}} \begin{bmatrix}
    % -c_\beta & 0\\
    % -\Pi e & \Pi \qty(e(x(s')^\top - x(s)^\top)) \end{bmatrix}\\
    =\begin{bmatrix} -c_\beta & \0 \\ -\Pi \E_{(s, a, s', e) \sim d_\fY}[e] & \Pi \overline A \end{bmatrix}
    =\begin{bmatrix} -c_\beta & \0 \\ -\Pi \E_{d_\fY}[e] & \Pi X_1^\top D_\pi(P_\lambda - I)X_1 \end{bmatrix} \qq{(Lemma~\ref{lem:A bar})}.
\end{align}
For simplicity, define $q \doteq \E_{d_\fY}[e], B \doteq X_1^\top D_\pi(P_\lambda - I)X_1$.
We then have
\begin{align}
    z^\top \widetilde{A} z = \begin{bmatrix} z_1 & z_2^\top \end{bmatrix} \begin{bmatrix} -c_\beta z_1 \\ \Pi  (-q z_1 + B z_2) \end{bmatrix} = -c_\beta z_1^2 + z_2^\top \Pi  (-q z_1 + B z_2).
\end{align}
Recall that $\Pi = X_1^\dagger X_1$ and it is symmetric, we can get
\begin{align}
    z_2^\top \Pi  (-q z_1 + B z_2) = (\Pi z_2)^\top (-q z_1 + B z_2) = z_2^\top (-q z_1 + B z_2),
\end{align}
where the last equality holds because $z_2\in \ker(X_1^\perp)$.
Thus,
    \begin{equation}
        z^\top \widetilde{A} z = -c_\beta z_1^2 - z_2^\top q z_1 + z_2^\top B z_2.
    \end{equation}
We now characterize $z_2^\top B z_2$.
Apparently, $z_2^\top B z_2 \leq 0$ always holds because $D_\pi(P_\lambda - I)$ is n.s.d.
In view of~\eqref{eq nsd solution},
the equality holds only if $X_1 z_2 = c\1$.
But $\1 \notin \col(X_1)$ and $z_2 \in \ker(X_1)^\perp$.
So the equality holds only when $z_2 = 0$.
Now we have proved that $\forall z_2 \in \ker(X_1)^\perp, z_2 \neq 0$, it holds that $z_2^\top B z_2 < 0$.
Using the normalization trick and the extreme value theorem again (cf.~\eqref{eq evt trick}),
we confirm that there exists some constant $C_{\tref{lem negdef pA}, 1} > 0$ such that $\forall z_2 \in \ker(X_1)^\perp$,
\begin{align}
    z_2^\top B z_2 \leq - C_{\tref{lem negdef pA}, 1} \norm{z_2}^2.
\end{align}

Since $z \neq 0$, we now discuss two cases. \\
\tb{Case 1: $z_1 = 0, z_2 \neq 0$.} In this case, we have $z^\top \widetilde A z = z_2^\top B z_2 < 0$. \\
\tb{Case 2: $z_1 \neq 0$.}
In this case, we have
\begin{align}
    z^\top \widetilde A z = -c_\beta z_1^2 + z_1 z_2^\top q + z_2^\top B z_2
    \leq -c_\beta z_1^2 + \abs{z_1} \norm{z_2} \norm{q} - C_{\tref{lem negdef pA}, 1} \norm{z_2}^2.
\end{align}
By completing squares,
it is easy to see that when $c_\beta$ is sufficiently large (depending on $\norm{q}$ and $C_{\tref{lem negdef pA}, 1}$),
it holds $z^\top \widetilde A z < 0$ because $z_1 \neq 0$.

Combining both cases,
we have proved that $\forall z \in \R \times \ker(X_1)^\perp, z \neq \0$,
it holds that
\begin{align}
    z^\top \widetilde A z < 0.
\end{align}
Using the normalization trick and the extreme value theorem again (cf.~\eqref{eq evt trick}) then completes the proof.
\end{proof}

\subsection{Proof of Lemma~\tref{lem gamma bound ar}}
\label{proof gamma bound ar}
\begin{proof}
By definition, $\widetilde W_* = \left\{ \begin{bmatrix}
    J_\pi \\ \Pi  w
\end{bmatrix} \bigg| w \in \overline W_* \right\}$. 
In view of Lemma~\ref{lem:fix_points},
let $\overline w_*$ be any fixed vector in $\overline W_*$.
Then any $\tw_* \in \widetilde W_*$ can be written as
\begin{align}
    \tw_* = \mqty[J_\pi \\ \Pi(\overline w_* + w_\0)]
\end{align}
with some $w_\0 \in \ker(X_1)$.
We then have
\begin{align}
    \widetilde X \tw_* = \mqty[J_\pi \\ X\Pi(\overline w_* + w_\0)] = \mqty[J_\pi \\ X\Pi\overline w_*],
\end{align}
where the last equality holds because $\Pi$ is the orthogonal projection to $\ker(X_1)^\perp$.
This means that $\widetilde X \tw_*$ is a constant regardless of $\tw_*$,
which completes the proof.
\end{proof}

\subsection{Proof of Lemma~\ref{lem:dd}}
\begin{lemma}
    \label{lem:dd}
        $(\hat{J_t}-J_\pi)^2+d(w_t, \overline W_*)^2 = d(\tw_t, \widetilde W_*)^2$.
    \end{lemma}
    \begin{proof}
        We recall that $\Pi$ is the orthogonal projection to $\ker(X_1)^\perp$.
        Let $\Pi'$ be the orthogonal projection to $\ker(X_1)$. 
        We recall from Lemma~\ref{lem:fix_points} that $\overline W_* = \qty{\overline w_*} + \ker(X_1)$ with $\overline w_*$ being any fixed point in $\overline W_*$.
        Thus for any $w_* \in \overline W_*$,
        we can write it as $\overline w_* + w_\0$ with some $w_\0 \in \ker(X_1)$.
        Then for any $w \in \R[d]$,
        we have
        \begin{align}
            d(w, \overline W_*)^2 =& \inf_{w_* \in \overline W_*} \norm{w - w_*}^2 \\
            =& \inf_{w_\0 \in \ker(X_1)} \norm{w - \overline w_* - w_\0}^2 \\
            =& \inf_{w_\0 \in \ker(X_1)} \norm{\Pi w + \Pi' w - \Pi \overline w_* - \Pi' \overline w_* - w_\0}^2 \\
            =& \inf_{w_\0 \in \ker(X_1)} \norm{\Pi w - \Pi \overline w_*}^2 + \norm{\Pi' w  - \Pi' \overline w_* - w_\0}^2 \\
            =& \norm{\Pi w - \Pi \overline w_*}^2,
        \end{align}
        where the last equality holds because we can select $w_\0 = \Pi' w  - \Pi' \overline w_*$.
        Define $\Pi \overline W_* \doteq \qty{\Pi w | w \in \overline W_*}$.
        Then we have
        \begin{align}
            d(\Pi w, \Pi \overline W_*) =& \inf_{w_* \in \overline W_*} \norm{\Pi w - \Pi w_*} \\
            =& \inf_{w_\0 \in \ker(X_1)} \norm{\Pi w - \Pi(\overline w_* + w_\0) } \\
            =& \norm{\Pi w - \Pi\overline w_*},
        \end{align}
        where the last equality holds because $w_\0 \in \ker(X_1)$ and $\Pi$ is the projection to $\ker(X_1)^\perp$ so $\Pi w_\0 = 0$.
        We now have $\forall w, d(w, \overline W_*) = d(\Pi w, \Pi\overline W_*)$.
        Then we have 
        \begin{align}
            &d(\widetilde w_t, \widetilde W_*)^2 \\
            =& (\hat{J_t}-J_\pi)^2 + d(\tw_t, \Pi \overline W_*)^2 \\ 
            =& (\hat{J_t}-J_\pi)^2 + d(\Pi w_t, \Pi \overline W_*)^2 \\ 
            =& (\hat{J_t}-J_\pi)^2 + d(w_t, \overline W_*)^2,
        \end{align}
        which completes the proof.
    \end{proof}

\section{Details of Experiments}
\label{sec:detail experiments}
% We validate the convergence guarantees of Linear TD($\lambda$) with arbitrary features (Theorems~\ref{thm:td markov} and \ref{thm:ar td markov}) on a modified 
We use a variant of Boyan's chain \citep{boyan1999least} with 15 states ($s_0, s_1, \dots, s_{14}$) and 5 actions ($a_0, \dots, a_4$). 
% Experiments cover both discounted and average-reward settings, focusing on the impact of parameters like $\gamma$, $\lambda$, and $\alpha_0$.
% \subsection{Experimental Setup}
The chain has deterministic transitions.
For $s_2, \dots, s_{14}$, 
the action $a_0$ goes to $s_{i-1}$ and the actions $a_1$ to $a_4$ go to $s_{i-2}$; $s_1$ always transitions to $s_0$; $s_0$ transitions uniformly randomly to any state. 
The reward function is
\begin{align}
    r(s, a) = \begin{cases}
        1 \qq{if} s=s_0 \\
        0 \qq{otherwise}
    \end{cases}.
\end{align}
We use a uniform random policy $\pi(a|s) = 0.5$.
The feature matrix $X \in \mathbb{R}^{15 \times 5}$ is designed to be of rank 3.
\begin{align*}
X = 
\begin{bmatrix}
0.07 & 0.11 & 0.18 & 0.14 & 0.61 \\
0.13 & 0.19 & 0.32 & 0.26 & 0.45 \\
0.11 & 0.17 & 0.28 & 0.22 & 0.39 \\
0.24 & 0.36 & 0.60 & 0.48 & 0.84 \\
0.18 & 0.28 & 0.46 & 0.36 & 1.00 \\
0.20 & 0.30 & 0.50 & 0.40 & 1.06 \\
0.31 & 0.47 & 0.78 & 0.62 & 1.45 \\
0.29 & 0.45 & 0.74 & 0.58 & 1.39 \\
0.42 & 0.64 & 1.06 & 0.84 & 1.84 \\
0.40 & 0.62 & 1.02 & 0.80 & 1.78 \\
0.47 & 0.73 & 1.20 & 0.94 & 2.39 \\
0.53 & 0.81 & 1.34 & 1.06 & 2.23 \\
0.58 & 0.9 & 1.48 & 1.16 & 2.78 \\
0.60 & 0.92 & 1.52 & 1.20 & 2.84 \\
0.67 & 1.03 & 1.70 & 1.34 & 3.45
\end{bmatrix}
\end{align*}
% with $X^\top D_\pi X$ having a condition number of 12.3.
% The stationary distribution $d_\pi$ is uniform ($d_\pi(s) = 1/15$).
Each experiment runs for $1.5\times10^6$ steps, averaged over $10$ runs. These experiments were conducted on a server equipped with an AMD EPYC 9534 64-Core Processor, with each run taking approximately $1$ minute to complete. Memory requirements are negligible.

\end{document}